\documentclass[12pt,reqno]{amsart}
\usepackage{amsmath, amssymb,fullpage}
\usepackage{graphicx, color}
\usepackage{enumitem,colonequals}
\usepackage[colorlinks=true,linkcolor=blue,citecolor=blue,urlcolor=blue]{hyperref}

\newtheorem{theorem}{Theorem}[section]
\newtheorem{lemma}[theorem]{Lemma}

\newtheorem{definition}[theorem]{Definition}
\newtheorem{remark}[theorem]{Remark}

\newcommand{\R}{\mathbb{R}}
\newcommand{\Z}{\mathbb{Z}}
\renewcommand{\P}{\mathbb{P}}

\newcommand{\proj}{\Pi}
\newcommand{\eps}{\varepsilon}
\renewcommand{\epsilon}{\varepsilon}

\title{On the Convergence of Adam, Revisited}
\author{Steven Heilman}
\author{Sampad Mohanty}
\date{\today}
\thanks{
Email: stevenmheilman@gmail.com, sbmohant@usc.edu\\
S.H. is supported by NSF Grant CCF AF 2448108.  
Department of Mathematics, University of Southern California, Los Angeles, CA 90089}

\begin{document}

\begin{abstract}
We show that projected Adam for online optimization with arbitrary moment decay parameters $\beta_1,\beta_2\in[0,1)$ can have average regret bounded away from zero.  A similar result of Reddi-Kale-Kumar from 2018 required $\beta_1<\sqrt{\beta_2}$.  Similar to their result, we use a three-periodic sequence of linear functions on $[-1,1]$ with slopes $c,-1,-1$, though we use $c$ slightly larger than $2$.  This nonzero average regret result extends to Adam variants such as AdamW, RMSProp, NAdam, Adan, AdaMax, Muon, and to an i.i.d. variant of the three-periodic sequence of slopes for Adam.
\end{abstract}

\maketitle

\section{Introduction}

In online minimization on $[-1,1]$, we are presented with a sequence of functions $f_1, f_2,\ldots$ where $f_{t}\colon[-1,1]\to\R$ for all $t\geq 1$.  At time $t\geq1$, we know $f_1(x_1),\ldots,f_{t-1}(x_{t-1})$ and $f_1'(x_1),\ldots,f_{t-1}'(x_{t-1})$, and we produce $x_t\in[-1,1]$.  For a fixed time horizon $T\geq1$, the goal is to minimize the regret $R_T$ at time $T$ against the best fixed comparator in $[-1,1]$, where
\begin{equation}\label{regretdef}
    R_T\colonequals\sum_{t=1}^T f_t(x_t)-\min_{x\in[-1,1]}\sum_{t=1}^T f_t(x).
\end{equation}

In contemporary applications, $f_t$ often depends on the $t$-th portion (or batch) of a large dataset.  The most popular optimization method for these applications is Adam \cite{kingma2015adam} (Adaptive Moment Estimation).  Adam and its variants perform online optimization to train neural networks, transformers, large language models, etc.  With nearly 250,000 citations, \cite{kingma2015adam} is currently one of the all time most highly cited scientific papers.

Under additional assumptions such as varying its parameters, Adam is known to converge, in the sense that $R_T /T\to0$ as $T\to\infty$ \cite{reddi2018adam}.  However, it is also known that Adam might not converge, i.e. there are examples of sequences of fairly reasonable functions $f_1,f_2,\ldots$ where projected Adam produces $x_1,x_2,\ldots$ with $R_T/T$ not converging to $0$ as $T\to\infty$.  However, these results only apply with restrictions on Adam's parameters \cite{reddi2018adam}.  In order to understand these parameters, let us define projected Adam.  

\begin{definition}[\textbf{Adam Optimization Method {\cite{kingma2015adam}}}]\label{adamopt}
Fix
\[
    b\colonequals\beta_1\in[0,1),\qquad q\colonequals\beta_2\in[0,1),\qquad \eps\ge0,
    \qquad \alpha_t>0,\quad\forall\,t\geq1.
\]
Let $x_1\in[-1,1]$ be arbitrary.  Define $x_{2},x_{3},\ldots\in[-1,1]$ as follows.  
\begin{equation}\label{adameq}
    m_t\colonequals bm_{t-1}+(1-b)g_t,
    \qquad
    v_t\colonequals qv_{t-1}+(1-q)g_t^2,\qquad
    g_t\colonequals f_t'(x_t),\qquad\forall\,t\geq1
\end{equation}
with the standard initialization $m_0=v_0=0$.  The projected update with step sizes $\alpha_{t}>0$ is
\begin{equation}\label{adamdef}
    x_{t+1}
    \colonequals\proj_{[-1,1]}\left(x_t-\alpha_t h_t\right),
    \qquad
    h_t\colonequals\frac{m_t}{\sqrt{v_t}+\eps},\qquad\forall\,t\geq1,
\end{equation}
where $\Pi_{[-1,1]}(x)\colonequals -1_{\{x<-1\}}+x1_{\{-1\leq x\leq 1\}}+1_{\{x>1\}}$ is projection of $x\in\R$ to the nearest element of $[-1,1]$.  Here $\alpha_t$ is called the learning rate or step size. 
For example, one could use $\alpha_{t}\colonequals\alpha/\sqrt{t}$ for some $\alpha>0$.  Also, if $\eps=0$, then $h_t$ is only defined when $v_t\neq0$.
\end{definition}

Some authors may refer to Adam as the above optimization method, but with no projection term $\proj_{[-1,1]}$ appearing  \eqref{adamdef}.  We will not do that.  Unless otherwise stated, we only refer to Adam as the method defined in Definition \ref{adamopt}.

\begin{remark}
We define RMSProp to be the Adam optimization method with $\beta_1=0$.  Other implementations called RMSProp may include momentum, centering, different epsilon
placement, or bias corrections; those variants require separate notation, although the same
short-memory denominator mechanism often persists.
\end{remark}
\begin{remark}
We briefly contrast Adam with  other optimization methods:
\begin{itemize}
\item $x_{t+1}\colonequals x_t - \alpha_t g_t$ (Gradient Descent) 
\item
$m_{t}\colonequals b m_{t-1}+g_t$,
$x_{t+1}\colonequals x_t - \alpha_t m_t$  (Heavy Ball)
\item $m_{t}\colonequals b m_{t-1}+f_t'(x_t - \alpha_t m_{t-1})$,
$x_{t+1}\colonequals x_t - \alpha_t m_t$ (Nesterov Accelerated Gradient)
\item $m_t\colonequals b m_{t-1}+(1-b)g_t$, $v_t\colonequals \max(q v_{t-1}, |g_t|)$, $h_t\colonequals\frac{m_t}{v_t + \epsilon}$, $x_{t+1}$ as in \eqref{adamdef} (AdaMax)
\item Same as Adam with $x_{t+1}
    \colonequals\proj_{[-1,1]}\left((1-\lambda\alpha_t)x_t-\alpha_t h_t\right)$ for some $\lambda\geq0$ (AdamW)
\item Same as Adam, but with $h_t\colonequals\frac{\beta_1 m_t+(1-\beta_1)g_t}{\sqrt{v_t}+\epsilon}$ (NAdam)
\item Same as Adam, with $q$ changing over time (NosAdam)
\end{itemize}
\end{remark}
%

%
%
%

The main parameters that can adjust the behavior of Adam are $\beta_1$ and $\beta_2$.  From the recursion \eqref{adameq}, we see that $\beta_1$ quantifies the amount of exponentially decaying ``memory'' of past derivatives of $f_t$ (where $b=\beta_1$ close to $1$ is a ``larger'' amount of such memory), since $m_t$ is approximately a function of $1/\log(1/b)$ previous time steps.  Likewise, $q=\beta_2$ quantifies the amount of ``memory'' of past squared gradients of $f_t$.

Here are some cited examples of Adam used to train large language models, together with their parameter descriptions.
\begin{itemize}
\item BERT was trained with Adam ``with learning rate of $10^{-4}$, $\beta_1 = 0.9$, $\beta_2 = 0.999$, L2 weight decay of 0.01, learning rate warmup over the first 10,000 steps, and linear decay of the learning rate.'' \cite{devlin19}.
\item GPT-3 was trained with Adam ``with $\beta_1 = 0.9$, $\beta_2 = 0.95$, and $\epsilon=10^{-8}$, we clip the global
norm of the gradient at 1.0, and we use cosine decay for learning rate down to 10\% of its value'' \cite{brown20}.
\item Llama 2 was trained ``using the AdamW optimizer \cite{los19}, with $\beta_1=0.9$, $\beta_2=0.95$, $\epsilon=10^{-5}$. We use a cosine learning rate schedule, with warmup of 2000 steps, and decay final learning rate down to 10\% of the peak learning rate. We use a weight decay of 0.1 and gradient clipping of 1.0.'' \cite{touvron23}.
\item DeepSeek-V3 ``employ[s] the AdamW optimizer \cite{los19} with hyper-parameters set to $\beta_1=0.9$, $\beta_2=0.95$ and weight decay $0.1$.'' \cite{deep25}
\end{itemize}

Despite the empirical success of Adam, it is known that it might not converge to its optimum.  The main result of Reddi, Kale, and Kumar~\cite{reddi2018adam} showed that Adam might not converge to its optimum for a sequence of linear functions on $[-1,1]$.

\begin{theorem}[{\cite[Theorem 2]{reddi2018adam}}]\label{rthm2}
Let $\beta_1<\sqrt{\beta_2}$, let $\alpha>0$ and let $\alpha_t\colonequals\alpha/\sqrt{t}$, for all $t\geq1$.  Then there exists a sequence of functions $f_1,f_2,\ldots\colon[-1,1]\to\R$ such that the Adam optimization method has regret satisfying: $R_T /T$ does not converge to zero as $T\to\infty$.
\end{theorem}

The example used was $f_t(x)=-x$ for all $t\geq1$ except $t\text{ mod }c=1$, in which case $f_t(x)=cx$, for all $x\in[-1,1]$.  That is, the slope of $f_t$ is $c$-periodic, where $c$ is chosen to be a sufficiently large number, as a function of $\beta_1,\beta_2$.  The idea is that the large positive slope that appears once is sufficient to offset the other smaller negative slopes.

Theorem \ref{rthm2} was also extended \cite[Theorem 3]{reddi2018adam} to the setting where the $f_t$ have random dependence on $t$.  That is, $f_t(x)=-x$ with probability $1-p$, and $f_t(x)=cx$ with probability $p$ for some appropriate $0<p<1$, with $f_1,f_2,\ldots$ i.i.d. random functions.  In \cite[Theorem 5]{reddi2018adam}, it is also shown that Adam can converge to its optimum if the parameters $\beta_1,\beta_2$ change over time.

As pointed out in \cite{reddi2018adam}, the paper that introduced Adam \cite[Corollary 4.2]{kingma2015adam} mistakenly claimed that Adam does converge, i.e. it has $R_T/T$ converging to zero as $T\to\infty$.  Investigating this issue then led to Theorem \ref{rthm2}.

Note that in the above four examples of BERT, GPT-3, Llama 2 and DeepSeek-V3, they already choose $\beta_1<\sqrt{\beta_2}$, i.e. they choose parameters where Theorem \ref{rthm2} applies.

Nevertheless, the results of \cite{reddi2018adam} left open the question of the existence of similar counterexamples for $\beta_1\geq\sqrt{\beta_2}$.  Moreover, the choice of slope $c$ can be arbitrarily large when $\beta_1$ or $\beta_2$ are close to $1$, i.e. $c\approx \max(1/\log(1/\beta_1),1/\log(1/\beta_2))$ is required in \cite{reddi2018adam}.  So, it was not clear if an example for Adam with nonzero average regret could be constructed with uniformly bounded slopes, even when $\beta_1<\sqrt{\beta_2}$.

\subsection{Our Contribution}

In this work we provide such a family of examples with nonzero average regret for Adam for all parameters $\beta_1,\beta_2\in[0,1)$ and with uniformly bounded gradients.

\begin{theorem}[Main]\label{thm:main}
Let $\alpha_{t}>0$ with $\lim_{t\to\infty}\alpha_t=0$, $\lim_{t\to\infty}\frac{\alpha_{t+1}}{\alpha_t}=1$, $\sum_{t=1}^{\infty}\alpha_t=\infty$.
$\forall$ $\beta_1,\beta_2\in[0,1)$, $\eps\geq0$, $\exists$ $f_1,f_2,\ldots\colon[-1,1]\to\R$ with $1\leq|f_t'(x)|\leq3$, $\forall$ $t\geq1,x\in[-1,1]$ such that Adam has regret satisfying: $R_T /T$ does not converge to zero as $T\to\infty$.
\end{theorem}

The example we use is simply $f_t(x)=-x$ for all $t\geq1$ except $t\text{ mod }3=1$, in which case $f_t(x)=(2+\delta)x$, for all $x\in[-1,1]$, where $\delta>0$ is chosen to be sufficiently small, depending on $\beta_1,\beta_2,\epsilon$.  That is, the slope of $f_t$ is $3$-periodic.  

Since $1\leq|f_t'(x)|\leq3$ for all $t\geq1$, $x\in[-1,1]$, the derivatives of the functions are uniformly bounded above and below, for all $\beta_1,\beta_2$.  

This same example showed nonzero regret of the $\beta_1=0$ case (known as RMSProp) of Adam in \cite[Theorem 6]{reddi2018adam} and \cite[Lemma 1]{huang2019}, inspiring Theorem \ref{thm:main}.

Despite the similarity of our example to the one from \cite{reddi2018adam}, our analysis is different and arguably simpler.  

As in \cite[Theorem 6]{reddi2018adam} in the $\beta_1=0$ case of Adam, we show that every three iterations of Adam leads to a net positive movement of $x_1,x_2,\ldots$ towards the point $x=1$, whereas the regret minimizer is $x=-1$.  However, we depart from \cite{reddi2018adam} by using an elementary fixed point argument via the contractive mapping theorem.  A related perspective was used in \cite{bockweiss2022adam}, albeit for quadratic functions.

This example also shows nonzero average regret for AdamW, RMSProp, NAdam, Adan, AdaMax, and Muon.

One might naturally ask if Theorem \ref{thm:main} holds when the highly structured periodic $f_1,f_2,\ldots$ is changed to a less structured i.i.d. variant of the above example, e.g. if for any $t\geq1$, $f_t(x)=ax$ with probability $1/3$, and $f_t(x)=-x$ with probability $2/3$, where $f_1,f_2,\ldots$ are all i.i.d.  We show the same nonzero average regret conclusion does hold in this case.  We present this result in the Appendix, Section \ref{secapp}.  Consequently, the $3$-periodicity of the example used in Theorem \ref{thm:main} is not required to obtain the theorem's conclusion.

The proof of Theorem \ref{thm:main} is  written for the uncorrected moments.  The same projected update with
standard bias-corrected moments, with
\[
    \widetilde m_t=\frac{m_t}{1-b^t},
    \qquad
    \widetilde v_t=\frac{v_t}{1-q^t},
    \qquad
    h_t\colonequals\frac{\widetilde m_t}{\sqrt{\widetilde v_t}+\eps},
\]
has the same asymptotic properties, since
$(1-b^t)^{-1}$ and $(1-q^t)^{-1}$ tend to one as $t\to\infty$.  Therefore the same asymptotic argument
applies to the bias-corrected case.

\subsection{Outline of Proof of Theorem \texorpdfstring{\ref{thm:main}}{Main}}

\begin{itemize}
\item Let $a\geq2$.  Let $f_{3k+1}(x)=ax$ and $f_{3k+2}(x)=f_{3k+3}(x)=-x$ for all $k\geq0$, $x\in\R$.
\item A contractive mapping argument shows $(m_{3k+1},m_{3k+2},m_{3k+3})$ and $(v_{3k+1},v_{3k+2},v_{3k+3})$ from \eqref{adameq} converge to $(M_1(a),M_2(a),M_3(a))$ and $(V_1(a),V_2(a),V_3(a))$, as $k\to\infty$.
\item Verify that the negative mean drift $S(a)\colonequals\sum_{i=1}^3 \frac{M_i(a)}{\sqrt{V_i(a)}+\eps}$ of $x_1,x_2,\ldots$ from three iterations of Adam, is negative when $a=2$.
\item A continuity argument shows, for $\delta>0$ small enough, $S(a)=S(2+\delta)<0$, so the negative mean drift is still negative for such $a$.
\item Conclude then that $\lim_{t\to\infty}x_{t}=1$.
\item Since $\sum_{t=1}^{3}f_{t}(x)=\delta x$, $\sum_{t=1}^{T}f_{t}(x)$ is minimized at $x=-1$ for $T$ large, so\\ $\lim_{T\to\infty}R_{T}/T=2\delta/3>0$, thereby completing the proof.
\end{itemize}

This argument is flexible enough to extend to other variants of Adam.

\begin{theorem}\label{thm:main2}
Theorem \ref{thm:main} also holds for: AdamW, NAdam, Adan, AdaMax and Muon
\end{theorem}

\begin{theorem}\label{thm:main3}
Let $\alpha>0$.  Then Theorem \ref{thm:main} holds almost surely for Adam with i.i.d. selection of the functions $f_1,f_2,\ldots$. and with step size $\alpha_t=\alpha/\sqrt{t}$ for all $t\geq1$
\end{theorem}

\subsection{Organization}

Theorem \ref{thm:main} will be proven in Section \ref{secmain}, by combining the previous Sections \ref{sec:steady-state}, \ref{secdrift} and \ref{secproj}.

Theorem \ref{thm:main2} will be stated more formally as separate versions of Theorem \ref{thm:main}, spread across the following sections: \ref{secadamw} for AdamW; \ref{secnadam} for NAdam; \ref{secadan} for Adan; \ref{secadamax} for AdaMax; and \ref{secmuon} for Muon.  Theorem \ref{thm:main3} is proven in Section \ref{secapp}.

\subsection{Further Discussion and Related Work}
\label{sec:related-work}

\subsubsection{Adam Alternatives such as AMSGrad}

Due to the convergence issues they found for Adam, 
Reddi et al.~\cite{reddi2018adam} proposed AMSGrad, which adds an additional parameter $\widehat v_t$ to Definition \ref{adamopt}, and then changes \eqref{adamdef} to
$$
    x_{t+1}
    \colonequals\proj_{[-1,1]}\left(x_t-\alpha_t h_t\right),
    \qquad
    h_t\colonequals\frac{m_t}{\sqrt{\widehat v_t}+\eps},\qquad
    \widehat v_t\colonequals \max(v_t, \widehat v_{t-1})
    \qquad\forall\,t\geq1,
$$
With this change, the previous periodicity issues for the squared gradient are removed.  AMSGrad then has provable regret bounds of the form $R_T = O(T^{1/2})$, so in particular $R_T /T \to0$ as $T\to\infty$ \cite{ala20}, assuming $\beta_1<\sqrt{\beta_2}$.  (\cite{reddi2018adam} also proved a regret bound of this form, but it needed to assume that $\beta_1$ decreased over time.)

Despite the superior theoretical guarantees of AMSGrad when compared to Adam, it appears that Adam is still more widely used in practice.

Subsequent variants, including
Yogi~\cite{zaheer2018yogi} and AdaBound/AMSBound~\cite{luo2019adabound}, were partly motivated by overcoming Adam's convergence issues found in \cite{reddi2018adam}.

\subsubsection{Adam divergence  with unbounded gradients}

In this work, we fix the parameters $\beta_1,\beta_2\in[0,1)$, and then produce an example of nonzero average regret for Adam with derivatives uniformly bounded above and below.  One might make these choices in the opposite order, i.e. fixing a function sequence (with possibly large derivatives) and then choosing $\beta_1,\beta_2$ to obtain a convergent method.  The latter perspective is taken in \cite{zhang2022adam,zhang2026adam}.  They show it is possible to choose $\beta_1,\beta_2$ (after the functions being optimized are fixed) such that Adam converges.

They also show that, for any $0\leq \beta_1,\beta_2<1$, there are functions such that Adam on the real line (without projection) diverges.  Their example \cite[Equation (3.1)]{zhang2026adam} is the following quadratic modification of \cite{reddi2018adam}: for any $x\in\R$, $a>0$, $1\leq i\leq n-1$, $n\geq4$,
$$
f_{0}(x)
\colonequals
\begin{cases}
(1+(n-1)a)x & \text{, if }x\geq-1\\
\frac{(1+(n-1)a)}{2}(x+2)^2 - \frac{3n}{2}& \text{, if }x<-1.\\
\end{cases}
\quad
f_{i}(x)
\colonequals
\begin{cases}
-ax & \text{, if }x\geq-1\\
-\frac{a}{2}(x+2)^2 +\frac{3}{2}& \text{, if }x<-1.\\
\end{cases}
$$
There are, however, some issues with this example, namely these functions are discontinuous unless $a=1$, and the proof of \cite[Theorem 3.5]{zhang2026adam} is only provided when $a=1$ and when the step size is constant in each training epoch.  These issues are fixable, but more importantly condition C1 \cite[Equation (8.3)]{zhang2026adam} seems to require $\beta_1<\sqrt{1-\beta_2}$, i.e. not all $\beta_1,\beta_2\in[0,1)$ are covered by their proof for the $a=1$ case; similarly, the suggested choice of $a=(n-1)^{-2}$ does not seem to allow all $\beta_1,\beta_2$ values in condition C1.  Also, condition C3 \cite[Equation (8.5)]{zhang2026adam} requires choosing a suitably small step size.  In any case, \cite[Theorem 3.5]{zhang2026adam} is incomparable to our Theorem \ref{thm:main} since their functions have quadratic components with unbounded gradients on an unbounded domain, whereas our functions have gradients bounded above and below on the bounded domain $[-1,1]$ with projection onto that domain.  Despite the above issues, the following modification should reproduce the result of \cite[Theorem 3.5]{zhang2026adam}: $f_{i}(x)=-x^{2}$ for $0<i<2n/3$ and $f_{i}(x)=16x^2$ for $2n/3\leq i\leq n-1$, where $n$ is chosen sufficiently large depending on $\beta_1,\beta_2$, since on the set $[0,\infty)$ we have $m_{i}/\sqrt{v_{i}}\approx -1$ for most $0<i<2n/3$ and $m_{i}/\sqrt{v_i}\approx1$ for most $2n/3\leq i\leq n-1$, so that $x_1,x_2,\ldots$ tends toward $+\infty$ while the true minimum occurs at $0$.

A different perspective for Adam is taken in Ahn, Zhang, Kook, and Dai~\cite{ahn2024adamftrl} where they interpret Adam as a discounted
Follow-the-Regularized-Leader method.

\subsubsection{Dynamical Systems Approach}
Da Silva and Gazeau~\cite{dasilva2020ode} derive a continuous-time ODE system for adaptive
first-order methods and analyze the convergence and stability of the limiting dynamics.  

Bai, Zhao, Zhou, Xu, and
Zhang~\cite{bai2026degenerate} study Adam on highly degenerate polynomials and give a
hyperparameter phase diagram containing stable convergence, spikes, and SignGD-like oscillation
regimes. These
papers concern related adaptive optimizers and stability phenomena, but not the bounded
online-linear regret setting of Theorem~\ref{thm:main}.

\subsubsection{Nonconvergence in traditional stochastic optimization frameworks}

The results below concern traditional stochastic optimization, instead of online optimization.

Wang and Klabjan~\cite{wang2022vradam} give stochastic divergence examples for Adam in
unconstrained strongly convex optimization, including examples that diverge in expectation or with
high probability and examples that persist for large mini-batches.  They also propose a variance-reduced
Adam-type method and prove convergence under a variance-reduction assumption.

Dereich, Graeber, and Jentzen~\cite{dereich2024nonvanishing} prove a nonconvergence
result for Adam and other adaptive stochastic-gradient methods when the learning rates are
asymptotically bounded away from zero.  

Dereich, Do, Jentzen, and von Wurstemberger~\cite{dereich2025symmetry} prove an Adam
symmetry theorem for stochastic strongly convex quadratic problems.  In their formulation, Adam
converges to the true minimizer if and only if the data distribution is symmetric.

Jentzen and
Riekert~\cite{jentzenriekert2025global} prove that Adam and SGD-type methods can fail with
high probability to converge to global minimizers in shallow ReLU-network training landscapes.
Do, Hannibal, and Jentzen~\cite{dohannibaljentzen2024relu} prove analogous high-probability
nonconvergence to global minimizers for a broad class of SGD methods, including Adam, in
data-driven supervised deep learning with ReLU activations.  Do, Jentzen, and Riekert~\cite{dojentzenriekert2025risk}
show nonconvergence of the true risk to the optimal risk for a large class of SGD-type methods,
again including Adam.

Toint~\cite{toint2023adam} gives a very simple deterministic one-dimensional example showing
that fixed-stepsize Adam can diverge on a smooth function with Lipschitz continuous gradient,
without gradient noise, irrespective of the method parameters.

\subsubsection{Contrast with NosAdam, AMSGrad, AdaGrad}

The one-dimensional counterexample we presented for the nonzero average regret of Adam and its relatives does not extend in a straightforward way to Adam variants with ``longer long-term memory'' such as AdaGrad, AMSGrad, NosAdam, etc. For example, instead of using the iteration for $v_t$ from \eqref{adameq}, AMSGrad keeps track of the maximum of $v_t$ with the additional parameter $\widehat{v}_t\colonequals\max(v_t, \widehat{v}_{t-1})$, and it then uses $h_{t}\colonequals \frac{m_t}{\sqrt{\widehat{v}_t}+\epsilon}$ in \eqref{adamdef}.  This eliminates the periodicity issue of $v_t$ that occurs for these counterexamples.  And indeed, these other methods often have better provable regret bounds than Adam.

\section{Steady-state Moments via Contraction}
\label{sec:steady-state}

We now prepare to prove Theorem \ref{thm:main}.  We first show the promised convergence of $m_{3k+i}$ and $v_{3k+1}$ as $k\to\infty$ using the contractive mapping theorem.  

Throughout this paper, we assume the gradients $g_{t}=f_{t}'(x_t)$ from \eqref{adameq} satisfy
\begin{equation}\label{gkdef}
    g_{3k+1}=a=2+\delta,
    \qquad
    g_{3k+2}=-1,
    \qquad
    g_{3k+3}=-1,\qquad\forall\,k\geq0
\end{equation}
where $\delta>0$ will be chosen sufficiently small.

\begin{lemma}\label{lemma1}
Assume \eqref{gkdef} holds.
Then there exist unique triples $(M_1,M_2,M_3)$ and $(V_1,V_2,V_3)$ that are fixed points of three iterations of \eqref{adameq}.  Moreover, $|m_{3k+i}-M_i|\leq O(b^{3k})$ and $|v_{3k+i}-V_i|\leq O(q^{3k})$ for all $k\geq0$, $1\leq i\leq 3$.
\end{lemma}
Note that, by \eqref{gkdef}, the iteration \eqref{adameq} does not depend on $x_t$.
\begin{remark}
We will show using elementary algebra that
\begin{equation}\label{eq:M-formulas}
    M_1(a)=\frac{a-b-b^2}{1+b+b^2},
    \qquad
    M_2(a)=\frac{ab-b^2-1}{1+b+b^2},
    \qquad
    M_3(a)=\frac{ab^2-b-1}{1+b+b^2}.
\end{equation}
\begin{equation}\label{eq:V-formulas}
    V_1(a)=\frac{a^2+q+q^2}{1+q+q^2},
    \qquad
    V_2(a)=\frac{a^2q+q^2+1}{1+q+q^2},
    \qquad
    V_3(a)=\frac{a^2q^2+q+1}{1+q+q^2}.
\end{equation}
Here we added the parameter $a$ to our notation to emphasize the dependence of $M_i,V_i$ on $a$.
\end{remark}
\begin{proof}
Let $M_1\in\R$.  Recall $g_{1}=a$, $g_{2}=g_3=-1$, by \eqref{gkdef}, so two iterations of \eqref{adameq} give
\begin{equation}\label{m2eq}
    M_2
    \stackrel{\eqref{adameq}}{=}
    bM_1-(1-b),\qquad t=2
\end{equation}
\begin{equation}\label{m3eq}
    M_3
    \stackrel{\eqref{adameq}}{=}
    bM_2-(1-b)\stackrel{\eqref{m2eq}}{=}b^2M_1-(1-b)(1+b),\qquad t=3.
\end{equation}
If we have a fixed point $(M_1,M_2,M_3)$, then \eqref{adameq} for $t=4$ should return to $M_1$, i.e. $M_1$ would be equal to (using $g_4=a$)
\[
    bM_3+(1-b)a
    \stackrel{\eqref{m3eq}}{=}
    b^3M_1+(1-b)(a-b-b^2).
\]
Thus the one-period return map for $M_1$ is the affine contraction $\Phi_b\colon\R\to\R$ defined by
\begin{equation}\label{phidef}
    \Phi_b(M)
    \colonequals
    b^3M+(1-b)(a-b-b^2),\qquad\forall\, M\in\R.
\end{equation}
 Since $b\in[0,1)$, we have $b^3<1$, so $\Phi_b$ has a unique fixed point by the contractive mapping theorem ($|\Phi_b (M) - \Phi_b (M')|\leq b^{3}|M-M'|<|M-M'|$ for all $M,M'\in\R$).  Similarly, $M_2,M_3$ are each the unique fixed point of a contraction, each of the form $M\mapsto b^{3}M+\text{constant}$.  Solving for $M_1$ in $\Phi_b(M_1)=M_1$ gives the first part of \eqref{eq:M-formulas}, then \eqref{m2eq} and \eqref{m3eq} yield the last part of \eqref{eq:M-formulas}.  The contraction property for $\Phi_b$ implies $|m_{3k+i}-M_i|\leq O(b^{3k})$, $\forall$ $k\geq0$, $1\leq i\leq 3$.

%
%

The argument for $(V_1, V_2, V_3)$ is analogous.  Since $g_t$ does not depend on $x_t$, three iterations of \eqref{adameq} for $v_t$ results in a contractive mapping of the form $\Psi_q(V)\colonequals q^{3}V+\text{ constant}$, i.e.
\begin{equation}\label{psidef}
    \Psi_q(V)
    \colonequals
    q^3V+(1-q)(a^2+q+q^2),\qquad\forall\,V\in\R.
\end{equation}
Since $q\in[0,1)$, this map has a unique fixed point $V_1$, by the contractive mapping theorem.  Solving for  $\Psi_q(V_1)=V_1$ produces the first equation in \eqref{eq:V-formulas}, and then \eqref{adameq} yields the last two parts of \eqref{eq:V-formulas}.   The contraction property implies $|v_{3k+i}-V_i|\leq O(q^{3k})$, $\forall$ $k\geq0$, $1\leq i\leq 3$.


\end{proof}

\section{Drift Away from the Minimizer}\label{secdrift}

In Lemma \ref{lemma1}, we found exponential convergence of the $m_{3k+i}$ and $v_{3k+i}$ terms from \eqref{adameq} to their limiting values as $k\to\infty$, $\forall$ $1\leq i\leq3$.  In this section, we then deduce the ``drift'' of the iterates $x_{t}$ themselves to the right endpoint $x=1$.  This ``drift'' will be quantified by
\begin{equation}\label{eq:S-def}
    S(a)\colonequals\sum_{i=1}^3 \frac{M_i(a)}{\sqrt{V_i(a)}+\eps}.
\end{equation}
If $S(a)<0$, then
the Adam update $x\mapsto x-\alpha_t h_t$ has positive net drift toward $x=1$.

\begin{lemma}
There exists $d\in(0,1)$ such that, for all $0\leq\delta\leq d$, $a\colonequals2+\delta$ satisfies
\begin{equation}\label{eq:delta-choice}
    S(a)<0,
    \qquad
    M_1(a)>0,
    \qquad
    M_2(a)<0,
    \qquad
    M_3(a)<0.
\end{equation}
\end{lemma}
\begin{proof}
Let $a=2$.  Since $b\in[0,1)$, we have by \eqref{eq:M-formulas} that
\begin{equation}\label{eq:M-at-two}
    M_1(2)=\frac{(1-b)(2+b)}{1+b+b^2}>0,
    \quad
    M_2(2)=-\frac{(1-b)^2}{1+b+b^2}<0,
    \quad
    M_3(2)=-\frac{(1-b)(1+2b)}{1+b+b^2}<0.
\end{equation}
Moreover,
\begin{equation}\label{msum}
    M_1(2)+M_2(2)+M_3(2)=0.
\end{equation}
For the second moment, we have by \eqref{eq:V-formulas}
\begin{equation}\label{eq:V-at-two}
    V_1(2)=\frac{4+q+q^2}{1+q+q^2},
    \qquad
    V_2(2)=\frac{1+4q+q^2}{1+q+q^2},
    \qquad
    V_3(2)=\frac{1+q+4q^2}{1+q+q^2}.
\end{equation}
Since $q<1$,
\[
    V_1(2)-V_2(2)
    \stackrel{\eqref{eq:V-at-two}}{=}\frac{3(1-q)}{1+q+q^2}>0,
\qquad
    V_1(2)-V_3(2)
    \stackrel{\eqref{eq:V-at-two}}{=}
    \frac{3(1-q^2)}{1+q+q^2}>0.
\]
Therefore
\begin{equation}\label{vineq}
    \sqrt{V_1(2)}+\eps>\sqrt{V_2(2)}+\eps,
    \qquad
    \sqrt{V_1(2)}+\eps>\sqrt{V_3(2)}+\eps.
\end{equation}
Let
\begin{equation}\label{abineq}
    A\colonequals-M_2(2)\stackrel{\eqref{eq:M-at-two}}{>}0,
    \qquad
    B:=-M_3(2)\stackrel{\eqref{eq:M-at-two}}{>}0.
\end{equation}
Since $M_1(2)=A+B$ by \eqref{msum}, we obtain
\begin{equation}\label{s2ineq}
\begin{aligned}
    S(2)
    &\stackrel{\eqref{eq:S-def}}{=}\frac{A+B}{\sqrt{V_1(2)}+\eps}
      -\frac{A}{\sqrt{V_2(2)}+\eps}
      -\frac{B}{\sqrt{V_3(2)}+\eps} \\
    &=A\left(
        \frac1{\sqrt{V_1(2)}+\eps}
        -\frac1{\sqrt{V_2(2)}+\eps}
      \right)
      +B\left(
        \frac1{\sqrt{V_1(2)}+\eps}
        -\frac1{\sqrt{V_3(2)}+\eps}
      \right)
    \stackrel{\eqref{vineq}\wedge\eqref{abineq}}{<}0.
\end{aligned}
\end{equation}
This strict negativity holds for every $b,q\in[0,1)$ and every $\eps\ge0$.

By continuity of $S(a)$ via \eqref{eq:S-def}, \eqref{eq:M-formulas} and \eqref{eq:V-formulas}, there exists $d>0$
such that \eqref{eq:delta-choice} holds for all $0\leq\delta\leq d$.  Replacing $d$ by $\min(d,.9)$ completes the proof.
\end{proof}

\section{A projection lemma}\label{secproj}

The net drift result of Section \ref{secdrift} does not immediately apply to Adam, due to the projection term $\Pi_{[-1,1]}$ in \eqref{adamdef}.  In this section, we therefore analyze this projection term applied thrice.

\begin{lemma}\label{lem:projection}
Let $P=\proj_{[-1,1]}$ so $P(x)=-1_{\{x<-1\}}+x1_{\{-1\leq x\leq 1\}}+1_{\{x>1\}}$ for all $x\in\R$.  If $u_1\le0$, $u_2\ge0$, $u_3\ge0$, and $U\colonequals u_1+u_2+u_3$, then
for every $x\in[-1,1]$,
\begin{equation}\label{peq}
    P\bigl(P(P(x+u_1)+u_2)+u_3\bigr)\ge P(x+U).
\end{equation}
Consequently, if $U\ge c>0$, then
\begin{equation}\label{peq2}
    P\bigl(P(P(x+u_1)+u_2)+u_3\bigr)\ge \min(1,x+c).
\end{equation}
\end{lemma}

\begin{proof}
Let $x\in[-1,1]$.  Since $x\le1$ and $u_1\le0$, we have $x+u_1\le1$, hence
\[
    P(x+u_1)=\max(-1,x+u_1)\ge x+u_1.
\]
Also, for any $z\in\R$ and any $w\ge0$,
\begin{equation}\label{p2}
    P(P(z)+w)\ge P(z+w).
\end{equation}
Indeed, if $z\le1$, then $P(z)\ge z$, and the claim follows from monotonicity of $P$; if
$z>1$, both sides are equal to $1$ since $P(z)=1$ and $w\geq0$.  Applying \eqref{p2} twice (using $u_2, u_3\geq0$)
\[
    P\bigl(P(P(x+u_1)+u_2)+u_3\bigr)
    \ge P(x+u_1+u_2+u_3)=P(x+U).
\]
So \eqref{peq} holds.  Now, assume $U\geq c>0$.  Then monotonicity of $P$ gives
\[
    P(x+U)\ge P(x+c)=\min(1,x+c),
\]
where the last equality uses $x\in[-1,1]$ so that $x+c\ge-1$.
\end{proof}

\section{Main theorem}\label{secmain}

We now prove Theorem \ref{thm:main}, restated as Theorem \ref{thm:all-betas} below.

\begin{theorem}\label{thm:all-betas}
Fix $b,q\in[0,1)$, $\eps\ge0$, and
$\alpha_t$ satisfying $\lim_{t\to\infty}\alpha_t=0$, $\lim_{t\to\infty}\alpha_{t+1}/\alpha_t=1$ and $\sum_{t=1}^{\infty}\alpha_t=\infty$.  There exists
$\delta>0$, depending only on $b,q$ and $\eps$, such that for every initial point
$x_1\in[-1,1]$, projected Adam on $[-1,1]$ applied to
\begin{equation}\label{ftdef}
    f_{3k+1}(x)=(2+\delta)x,
    \qquad
    f_{3k+2}(x)=-x,
    \qquad
    f_{3k+3}(x)=-x,
\end{equation}
satisfies
\begin{equation}\label{xt1}
    \lim_{t\to\infty}x_t=1.
\end{equation}
For every horizon $T\ge1$, the best fixed comparator in $[-1,1]$ is $x_T^*=-1$, and the
average regret satisfies
\begin{equation}\label{rt1}
    \lim_{T\to\infty}\frac{R_T}{T}=\frac{2\delta}{3}>0.
\end{equation}
\end{theorem}

\begin{proof}
Choose $\delta>0$ so that \eqref{eq:delta-choice} holds.  Let $a\colonequals2+\delta$.
Since the $m_{3k+i}$ and $v_{3k+i}$ terms converge as $k\to\infty$ by Lemma \ref{lemma1} $\forall$ $1\leq i\leq3$,  \eqref{adamdef} implies that the $h_{3k+i}$ terms also converge:
\begin{equation}\label{hidef}
    \lim_{k\to\infty}h_{3k+i}=H_i\colonequals\frac{M_i(a)}{\sqrt{V_i(a)}+\eps},
    \qquad \forall\,i=1,2,3.
\end{equation}
By \eqref{eq:delta-choice},
\begin{equation}\label{hineq}
    H_1>0,
    \qquad
    H_2<0,
    \qquad
    H_3<0,
    \qquad
    H_1+H_2+H_3=S(a)<0.
\end{equation}
Let
$
    \eta\colonequals-S(a)>0.
$
Define the three unprojected increments
\begin{equation}\label{ukdef}
    u_{k,i}\colonequals-\alpha_{3k+i}h_{3k+i},\qquad\forall\,k\geq0,
    \qquad i=1,2,3.
\end{equation}
Then for all $k$ sufficiently large, \eqref{hineq} and \eqref{hidef} imply
\begin{equation}\label{ukineq}
    u_{k,1}<0,
    \qquad
    u_{k,2}>0,
    \qquad
    u_{k,3}>0.
\end{equation}
Furthermore, using $\lim_{t\to\infty}\alpha_{t+1}/\alpha_t=1$,
\[
    U_k\colonequals u_{k,1}+u_{k,2}+u_{k,3}
    \stackrel{\eqref{ukdef}}{=}-\sum_{i=1}^3 \alpha_{3k+i}h_{3k+i} 
    \stackrel{\eqref{hidef}}{=}-\alpha_{3k+1}\cdot\left(\sum_{i=1}^3 H_i+o_k(1)\right)
      =\alpha_{3k+1}(\eta+o_k(1)).
\]
Consequently, there are constants $c>0$ and $k_0\ge1$ such that
\begin{equation}\label{eq:positive-period-drift}
    U_k\ge c\cdot \alpha_{3k+1}
    \qquad\text{for all } k\ge k_0.
\end{equation}

Let $z_k\colonequals x_{3k+1}$ be the iterate at the start of a period, for all $k\geq0$.  Applying Lemma~\ref{lem:projection} with $u_i=u_{k,i}$ for each $1\leq i\leq 3$ which is valid by \eqref{ukineq}
and \eqref{eq:positive-period-drift},
\begin{equation}\label{eq:period-recursion}
    z_{k+1}\stackrel{\eqref{adamdef}\wedge\eqref{ukdef}}{\ge} \min\Big(1,z_k+c\cdot\alpha_{3k+1}\Big),\qquad\forall\,k\geq k_0.
\end{equation}
Since $z_k\le1$ for all $k\geq1$ and $\sum_{k\ge k_0}\alpha_{k}=\infty$, we have $\sum_{k\geq k_0}\alpha_{3k+1}=\infty$, which follows since $\lim_{k\to\infty}\alpha_{k+1}/\alpha_{k}=1$. Then iterating
\eqref{eq:period-recursion} forces $z_k$ to equal $1$ after finitely many periods.  Indeed,
once $\sum_{k=k_0}^{k_1}c\cdot\alpha_{3k+1}>1-z_{k_0}$ for some $k_1>k_0$, \eqref{eq:period-recursion} gives $z_{k+1}\ge1$,
while projection onto $[-1,1]$ from \eqref{adamdef} gives $z_{k+1}\le1$.  Then $\forall$ $k>k_1$, \eqref{eq:period-recursion} gives
$z_{k+1}=1$ whenever $z_k=1$.

This implies convergence of the full sequence $(x_t)$ to $1$, since the within-period moves have magnitude $O(\alpha_{t})$.  To see this, note by \eqref{adameq} that
$(m_t)$ is bounded since $(g_t)$ is bounded, i.e. $|m_t|\leq \max(|m_{t-1}|,|g_t|)\leq3$ for all $t\geq0$.  Also, since $|g_t|\ge1$ for every $t$,
\[
    v_t\stackrel{\eqref{adameq}}{=}
    qv_{t-1}+(1-q)g_t^2\stackrel{\eqref{adameq}}{\ge} 1-q>0.
\]
Hence
\[
    |h_t|\stackrel{\eqref{adamdef}}{=}\Big|\frac{m_t}{\sqrt{v_t}+\eps}\Big|\leq\frac{3}{\sqrt{1-q}},\qquad\forall\,t\geq1.
\]
Therefore $|x_{t+1}-x_t|\leq |\alpha_t|\frac{3}{\sqrt{1-q}}$, $\forall$ $t\geq1$.  Since $x_{3k+1}\to1$ as $k\to\infty$ and $\lim_{t\to\infty}\alpha_{t}=0$, this implies that $x_t\to1$ as $t\to\infty$.  That is, \eqref{xt1} holds.

It remains to prove \eqref{rt1}.  For any $T\ge1$, let
\begin{equation}\label{gtdef}
    G_T\colonequals\sum_{t=1}^T g_t.
\end{equation}
Writing $T=3k+r$, $r\in\{0,1,2\}$, gives (recalling $a=2+\delta$)
\begin{equation}\label{gtobs}
    G_T\stackrel{\eqref{gkdef}}{=}
    \begin{cases}
        k\delta, & r=0,\\
        k\delta+a, & r=1,\\
        k\delta+a-1, & r=2.
    \end{cases}
\end{equation}
All three quantities are positive since $a=2+\delta$ and $\delta>0$.  Therefore the best fixed comparator is
always $x_T^*=-1$, and
\begin{equation}\label{mindo}
    \min_{x\in[-1,1]}\sum_{t=1}^T f_t(x)
    \stackrel{\eqref{ftdef}\wedge\eqref{gtdef}}{=}G_T\cdot\min_{x\in[-1,1]}x =-G_T.
\end{equation}
The regret is then
\begin{equation}\label{rtid}
R_T\stackrel{\eqref{regretdef}\wedge\eqref{ftdef}\wedge\eqref{mindo}}{=}\sum_{t=1}^T g_t x_t+G_T
    \stackrel{\eqref{gtdef}}{=}\sum_{t=1}^T g_t\cdot(x_t+1).
\end{equation}
Since $\lim_{t\to\infty}x_t =1$ by \eqref{xt1} and $(g_t)$ is bounded by \eqref{gkdef},
\[
    \lim_{T\to\infty}\frac1T\sum_{t=1}^T g_t\cdot(x_t-1)=0.
\]
Also $\lim_{T\to\infty}G_T/T=\delta/3$ by \eqref{gtobs}.  Hence
\[
    \frac{R_T}{T}
    \stackrel{\eqref{rtid}\wedge\eqref{gtdef}}{=}\frac{2G_T}{T}+\frac1T\sum_{t=1}^T g_t(x_t-1)
    \to\frac{2\delta}{3},\qquad\text{ as }T\to\infty.
\]
This proves \eqref{rt1} and completes the proof.
\end{proof}

\section{AdamW}\label{secadamw}

We now define AdamW and extend Theorem \ref{thm:main} to AdamW.

\begin{definition}[\textbf{AdamW Optimization Method}]\label{adamwdef}
Let $x_1\in[-4,-2]$ be arbitrary.  Define $x_{2},x_{3},\ldots\in[-4,-2]$ as follows.
The first and second moment recursions of AdamW~\cite{los19} with parameters $b=\beta_1$ and $q=\beta_2$ are the same as Adam, i.e. \eqref{adameq} holds,
with the standard initialization $m_0=v_0=0$.  The projected update then adds a single extra term $\lambda\geq0$ to \eqref{adamdef} as follows
\begin{equation}\label{adamweq}
    x_{t+1}
    \colonequals\proj_{[-4,-2]}\left((1-\lambda\alpha_t)x_t-\alpha_t h_t\right),
    \qquad
    h_t\colonequals\frac{m_t}{\sqrt{v_t}+\eps},\qquad\forall\,t\geq1.
\end{equation}
\end{definition}

\begin{theorem}[AdamW Counterexample]
Let $\alpha_t>0$ satisfy
$\lim_{t\to\infty}\alpha_t=0$, $\lim_{t\to\infty}\frac{\alpha_{t+1}}{\alpha_t}=1$ and $\sum_{t=1}^{\infty}\alpha_t=\infty$.
Fix $\beta_1,\beta_2\in[0,1)$, $\varepsilon\ge0$ and $\lambda\ge0$.
Consider projected AdamW on the domain $\mathcal F=[-4, -2]$.  Then there exists $\delta>0$, depending only on $b,q,\varepsilon$, such that for
the linear functions \eqref{ftdef}, 
the iterates of AdamW satisfy
\[
    \lim_{t\to\infty}x_t =-2.
\]
However, the best fixed comparator is $x^*=-4$, and
\[
    \lim_{T\to\infty}\frac{R_T}{T}
    =
    \frac{2\delta}{3}.
\]
\end{theorem}
\begin{proof}
Define $u_{k,i}$ from \eqref{ukdef}.  The corresponding
increments for AdamW are then
\[
    u_{k,i}-\lambda\alpha_{3k+i}x_{3k+i}.
\]
On $\mathcal F=[-4,-2]$, we have $x_{3k+i}\le -2$, and hence
\[
    -\lambda\alpha_{3k+i}x_{3k+i}\ge0.
\]

We now adapt the remaining parts of Theorem \ref{thm:all-betas} to AdamW.

Put \(\mathcal F=[-4,-2]\), and for any $t\geq1$, let
\[
u_t\colonequals-\alpha_th_t.
\]
For fixed $t$, $g_t$ is a constant that does not depend on $x_t$, so \((h_t)\) also does not depend on $x_t$.

Since $\lim_{t\to\infty}\alpha_t=0$, we may choose \(T=3K+1\) sufficiently large that
\[
0\le\lambda\alpha_t\le1,
\qquad \forall\,t\ge T.
\]
For any \(t\ge T\), define the one-step maps
\[
\mathcal A_t(x)\colonequals\Pi_{\mathcal F}(x+u_t),
\qquad
\mathcal W_t(x)
 \colonequals\Pi_{\mathcal F}\bigl((1-\lambda\alpha_t)x+u_t\bigr),\qquad\forall\,x\in[-4,-2].
\]
The map \(\mathcal W_t\) is nondecreasing since
\(1-\lambda\alpha_t\ge0\) and \(\Pi_F\) is nondecreasing.
Furthermore, for every \(x\in F\),
\[
(1-\lambda\alpha_t)x+u_t
 =
 x+u_t-\lambda\alpha_tx
 \ge x+u_t,
\]
since \(x\le-2<0\).  Therefore
\begin{equation}\label{waineq}
\mathcal W_t(x)\ge\mathcal A_t(x),
\qquad \forall\,x\in \mathcal F.
\end{equation}

Let \((y_t)_{t\ge T}\) be the auxiliary projected Adam sequence
without weight decay, initialized by
\[
y_T\colonequals x_T,
\qquad
y_{t+1}\colonequals\mathcal A_t(y_t),\qquad\forall\,t\geq T.
\]
The period-three argument from Theorem~\ref{thm:all-betas}, translated from
\([-1,1]\) to \(\mathcal F=[-4,-2]\), gives
$\lim_{t\to\infty}y_t=-2$.  We claim that
\[
x_t\ge y_t,
\qquad \forall\,t\ge T.
\]
This holds at time \(T\) by definition of $y_T$.  If it holds at any $t\geq T$, then
monotonicity of \(\mathcal W_t\) and \eqref{waineq} give
\[
x_{t+1}
 \stackrel{\eqref{adamweq}}{=}\mathcal W_t(x_t)
 \ge\mathcal W_t(y_t)
 \ge\mathcal A_t(y_t)
 =y_{t+1}.
\]
Thus the claim follows by induction.  Since both sequences lie in
\([-4,-2]\), $-2\geq x_t\geq y_t$ and $\lim_{t\to\infty}y_t=-2$ imply that $\lim_{t\to\infty}x_t=-2$.

The cumulative gradient over each period is still $\delta>0$, so the best
fixed comparator on $[-4,-2]$ is the left endpoint $-4$.  Since the iterates
converge to the right endpoint $-2$ and the interval length is $2$, the same
regret computation from \eqref{rt1} gives
$
    \lim_{T\to\infty}\frac{R_T}{T}=\frac{2\delta}{3}.
$
\end{proof}

\section{NAdam}\label{secnadam}

We now define NAdam and extend Theorem \ref{thm:main} to NAdam.
Recall that projected NAdam is defined exactly as in Definition \ref{adamopt}, but instead of the $h_t$ from \eqref{adamdef} we have
$$h_t\colonequals\frac{b m_t+(1-b)g_t}{\sqrt{v_t}+\epsilon},\qquad\forall\,t\geq1.$$

\begin{theorem}[NAdam Counterexample]
Let $\alpha_t>0$ satisfy
$\lim_{t\to\infty}\alpha_t=0$, $\lim_{t\to\infty}\frac{\alpha_{t+1}}{\alpha_t}=1$ and $\sum_{t=1}^{\infty}\alpha_t=\infty$.  Fix $b,q\in[0,1)$, $\varepsilon\ge0$.  Consider projected NAdam on the domain $\mathcal F=[-1,1]$.
Then there exists $\delta>0$, depending only on $b,q,\varepsilon$, such that for
the linear functions \eqref{ftdef}, 
the iterates of NAdam satisfy
\[
    \lim_{t\to\infty}x_t =1.
\]
However, the best fixed comparator is $x^*=-1$, and
\[
    \lim_{T\to\infty}\frac{R_T}{T}
    =
    \frac{2\delta}{3}.
\]
\end{theorem}

\begin{proof}
The recursions for $m_t$ and $v_t$ are identical to Adam, so Lemma \ref{lemma1} applies, and the limiting values $M_i,V_i$ from Lemma \ref{lemma1} are unchanged.  For NAdam with $a=2$, the limiting numerators of the $h$ terms are $U_{1},U_{2},U_{3}$ where
$$U_i = b M_{i}(2)+(1-b)g_i,\qquad\forall\,i\in\{1,2,3\}.$$

Since
$
    \sum_{i=1}^3 M_i(2)=0,
$
by \eqref{eq:M-formulas}
and
$
    \sum_{i=1}^3 g_i=0,
$
when $a=2$, 
we have $U_1+U_2+U_3=0$.  Moreover $U_1>0$ and
$U_2,U_3<0$, since both $M_i(2)$ and $g_i$ have these signs by \eqref{eq:delta-choice} and \eqref{gkdef}.  The remaining details follow those of Adam in Theorem \ref{thm:all-betas} with $S_{\mathrm{NAdam}}(a)\colonequals\sum_{i=1}^3 \frac{U_i(a)}{\sqrt{V_i(a)}+\eps}$, $U_{i}(a)\colonequals b M_i(a)+(1-b)g_i$ for all $i\in\{1,2,3\}$.  For example, $S_{\mathrm{NAdam}}(2)<0$ by repeating the proof of \eqref{s2ineq} mutatis mutandis, so $S_{\mathrm{NAdam}}(a)<0$ for all $a\in\R$ near $2$, and so on.
\end{proof}

\section{Adan}\label{secadan}

We define the Adan optimization method \cite{xie24} and extend Theorem \ref{thm:main} to it.  Let
\[
    b\colonequals\beta_1,\qquad q\colonequals\beta_2,\qquad r\colonequals\beta_3,\qquad b,q,r\in[0,1).
\]
Let $g_0\colonequals-1$, $m_{0}=d_0 = n_0\colonequals0$.  For any $t\geq1$, define
\[
    m_t\colonequals bm_{t-1}+(1-b)g_t,
\]
\[
    d_t\colonequals qd_{t-1}+(1-q)(g_t-g_{t-1}),
\]
\[
    n_t\colonequals r n_{t-1}
        +(1-r)\bigl(g_t+q(g_t-g_{t-1})\bigr)^2.
\]
\[
x_{t+1}
    \colonequals\proj_{[-1,1]}\left(x_t-\alpha_t h_t\right),
    \qquad
       h_t
    \colonequals
    \frac{m_t+qd_t}{\sqrt{n_t}+\varepsilon}.
\]
This definition in terms of decay coefficients $b,q,r$ may differ from other definitions.

\begin{theorem}[Adan Counterexample]
Let $\alpha_t>0$ satisfy
$\lim_{t\to\infty}\alpha_t=0$, $\lim_{t\to\infty}\frac{\alpha_{t+1}}{\alpha_t}=1$ and $\sum_{t=1}^{\infty}\alpha_t=\infty$.  Fix $b,q,r\in[0,1)$, $\varepsilon\ge0$.  Consider projected Adan on the domain $\mathcal F=[-1,1]$.
Then there exists $\delta>0$, depending only on $b,q,r,\varepsilon$, such that for
the linear functions \eqref{ftdef}, 
the iterates of Adan satisfy
\[
    \lim_{t\to\infty}x_t =1.
\]
However, the best fixed comparator is $x^*=-1$, and
\[
    \lim_{T\to\infty}\frac{R_T}{T}
    =
    \frac{2\delta}{3}.
\]
\end{theorem}

\begin{proof}
Since $g_t$ does not depend on $x_t$, three Adan iterations for $m_t$ or $d_t$ or $n_t$ results in a contractive mapping.  For example, $m_{3k+1}\mapsto m_{3(k+1)+1}$ corresponds to a map $M\mapsto b^{3}M+\text{constant}$.  Since this map is a contraction from $\R$ to $\R$, it has a unique fixed point, by the contractive mapping theorem.  Thus, $\lim_{k\to\infty}(m_{3k+1},m_{3k+2},m_{3k+3})\equalscolon(M_1, M_2, M_3)$, $\lim_{k\to\infty}(n_{3k+1},n_{3k+2},n_{3k+3})\equalscolon(N_1, N_2, N_3)$, $\lim_{k\to\infty}(d_{3k+1},d_{3k+2},d_{3k+3})\equalscolon(D_1, D_2, D_3)$.

The first-moment values from \eqref{eq:M-at-two} are
\begin{equation}\label{mdef2}
    M_1=\frac{2-b-b^2}{1+b+b^2},
    \quad
    M_2=\frac{2b-b^2-1}{1+b+b^2},
    \quad
    M_3=\frac{2b^2-b-1}{1+b+b^2}.
\end{equation}
The gradient-difference values are
\begin{equation}\label{ddef}
    D_1=\frac{3(1-q^2)}{1+q+q^2},
    \quad
    D_2=-\frac{3(1-q)}{1+q+q^2},
    \quad
    D_3=-\frac{3q(1-q)}{1+q+q^2}.
\end{equation}
Thus the limiting Adan numerators are
\begin{equation}\label{udef}
    U_i\colonequals M_i+qD_i,\qquad\forall\,i\in\{1,2,3\}.
\end{equation}
From \eqref{mdef2}, \eqref{ddef} and $q\geq0$, they satisfy
\[
    U_1>0,\qquad U_2<0,\qquad U_3<0,
    \qquad
    U_1+U_2+U_3=0.
\]

Write
\[
    A\colonequals-U_2>0,\qquad B\colonequals-U_3>0,
\]
so that \(U_1=A+B\).  Set
\[
    s_1\colonequals(2+3q)^2,\qquad s_2\colonequals(1+3q)^2,\qquad s_3\colonequals 1.
\]
Then
\[
    N_1=\frac{s_1+r^2s_2+rs_3}{1+r+r^2},
\qquad
    N_2=\frac{rs_1+s_2+r^2s_3}{1+r+r^2},
\qquad
    N_3=\frac{r^2s_1+rs_2+s_3}{1+r+r^2}.
\]
Define
\begin{equation}\label{wdef}
    w_i\colonequals\frac{1}{\sqrt{N_i}+\varepsilon},\qquad\forall\,i\in\{1,2,3\}.
\end{equation}
The limiting drift of three Adan iterations at \(a=2\) is
\begin{equation}\label{sadan2def}
    S_{\rm Adan}(2)
    =
    \sum_{i=1}^3 U_iw_i
    =
    A(w_1-w_2)+B(w_1-w_3).
\end{equation}

We always have \(N_1>N_3\) since $N_1 - N_3=(1-r)[(s_1 - s_3) +r(s_1 - s_2)]/(1+r+r^2)>0$.  If \(N_1\ge N_2\), then
\(w_1\le w_2\) and \(w_1<w_3\), so \(S_{\rm Adan}(2)<0\).

It remains to consider the case \(N_2>N_1>N_3\).  (In the case $q=0$, we have $s_1=4,s_2=1$, and $N_1 - N_2=3(1-r)/(1+r+r^2)>0$, i.e. this case cannot occur when $q=0$, so we may assume $q>0$.)  The function $\phi\colon[0,\infty)\to\R$ defined by
\[
    \phi(x)=\frac1{\sqrt{x}+\varepsilon},\qquad\forall\,x\geq0
\]
%
has \(-\phi'(x)\) positive and decreasing.  Hence using \(N_2>N_1>N_3\)
\begin{equation}\label{wijineq}
    \frac{w_3-w_1}{w_1-w_2}
    \stackrel{\eqref{wdef}}{=}\frac{\int_{N_3}^{N_1}-\phi'(x)dx}{\int_{N_1}^{N_2}-\phi'(x)dx}
    \ge
    \frac{N_1-N_3}{N_2-N_1}.
\end{equation}
A direct calculation gives
\begin{equation}\label{nijineq}
    \frac{N_1-N_3}{N_2-N_1}> \frac1q.
\end{equation}
Indeed, after canceling the common $(1+r+r^2)/(1-r)$ factor, this is
equivalent to
\[
    q\bigl((1+r)s_1-rs_2-s_3\bigr)
    >
    (1+r)s_2-s_1-rs_3,
\]
and the left side minus the right side is equal to
\[
    3\Bigl(3q^3+4q^2+3q+1-r(q^2+q)\Bigr)>0.
\]
%
 
On the other hand,
\[
    \frac{A}{B}\le \frac1q.
\]
This follows since, using $b,q\in[0,1)$ and $q\neq0$
\begin{equation}\label{mineq}
    \frac{-M_2}{-M_3}
    \stackrel{\eqref{mdef2}}{=}\frac{1-b}{1+2b}\le 1< \frac1q,
\end{equation}
while
\begin{equation}\label{dineq}
    \frac{-qD_2}{-qD_3}
    \stackrel{\eqref{ddef}}{=}
    \frac{1}{q}.
\end{equation}
Therefore
\[
    \frac{A}{B}
    =\frac{U_2}{U_3}
    \stackrel{\eqref{udef}}{=}\frac{-M_2 - qD_2}{-M_3 - qD_3}
\stackrel{\eqref{mineq}\wedge\eqref{dineq}}{\leq}
\frac{1}{q}
    \stackrel{\eqref{wijineq}\wedge\eqref{nijineq}}{<}
    \frac{w_3-w_1}{w_1-w_2}.
\]
Rearranging this inequality gives
\[
    A(w_1-w_2)-B(w_3-w_1)<0,
\]
so by \eqref{sadan2def} we get (in all cases) that
\[
    S_{\rm Adan}(2)<0.
\]

By continuity, there exists \(\delta>0\) such that, with \(a=2+\delta\),
the limiting Adan period sum $\sum_{i=1}^3 U_i(a)w_i(a)$ remains negative and the numerator signs remain
\[
    U_1(a)>0,\qquad U_2(a)<0,\qquad U_3(a)<0.
\]
The projection and regret argument from Theorem~\ref{thm:all-betas} then applies
verbatim.  
\end{proof}

\section{AdaMax}\label{secadamax}

We define the AdaMax optimization method and extend Theorem \ref{thm:main} to it. 
The AdaMax optimization method is defined by the following recursions:

$m_t\colonequals b m_{t-1}+(1-b)g_t$, $v_t\colonequals \max(q v_{t-1}, |g_t|)$, $h_t\colonequals\frac{m_t}{v_t + \epsilon}$, $x_{t+1}$ as in \eqref{adamdef}, for all $t\geq1$.

As usual, $m_0=v_0=0$.

\begin{theorem}[AdaMax Counterexample]
Let $\alpha_t>0$ satisfy
$\lim_{t\to\infty}\alpha_t=0$, $\lim_{t\to\infty}\frac{\alpha_{t+1}}{\alpha_t}=1$ and $\sum_{t=1}^{\infty}\alpha_t=\infty$.  Fix $b,q\in[0,1)$, $\varepsilon\ge0$.  Consider projected AdaMax on the domain $\mathcal F=[-1,1]$.
Then there exists $\delta>0$, depending only on $b,q,\varepsilon$, such that for
the linear functions \eqref{ftdef}, 
the iterates of AdaMax satisfy
\[
    \lim_{t\to\infty}x_t =1.
\]
However, the best fixed comparator is $x^*=-1$, and
\[
    \lim_{T\to\infty}\frac{R_T}{T}
    =
    \frac{2\delta}{3}.
\]
\end{theorem}

\begin{proof}
The steady-state values $(M_1, M_2, M_3)$ from Lemma \ref{lemma1} apply here as well, since the $m_t$ recursion for AdaMax is the same as Adam.  To find the steady-state $(V_1,V_2,V_3)$ we plug them into their recursion as follows:
$$V_2 = \max(qV_1, 1)$$
$$V_3 = \max(qV_2,1)=\max(q^2 V_1,1)$$
One more iteration gives
$$\max(q^3 V_1, a)$$
So the fixed point of $V\mapsto \max(q^3 V,a)$ must satisfy $V=a$.  (Note that $\Phi(V)\colonequals\max(q^{3}V,a)$ satisfies $|\Phi(V)-\Phi(V')|\leq q^{3}|V-V'|$ for all $V,V'\in\R$ and $0\leq q<1$ implies that $\Phi$ is a contraction, so $\Phi$ has a unique fixed point.)  The corresponding maps for $V_2,V_3$ are also contractions.  Then solving for $V_2,V_3$ gives
\begin{equation}\label{vmaxdef}
V_1(a)=a,\quad V_2(a)=\max(qa,1),\quad V_3(a)=\max(q^2a,1),
\end{equation}
and when $a=2$ we have
$$V_1 = 2,\quad V_2 = \max(2q,1),\quad V_3 = \max(2q^2, 1).$$

Since $q<1$ we therefore have
\begin{equation}\label{vineq2}
    V_1>V_2,\qquad V_1>V_3,
\end{equation}
For any $a\geq2$, define
$S_{\text{AdaMax}}(a)\colonequals\sum_{i=1}^{3}\frac{M_i(a)}{V_i(a)+\epsilon}
$.  We then have
$$S_{\text{AdaMax}}(2)=\sum_{i=1}^{3}\frac{M_i(2)}{V_i(2)+\epsilon}.$$

Let $A,B>0$ as in \eqref{abineq}.  Since $M_1(2)=A+B$ by \eqref{msum}, we obtain
\begin{align*}
    S_{\text{AdaMax}}(2)
    &=\frac{A+B}{V_1(2)+\eps}
      -\frac{A}{V_2(2)+\eps}
      -\frac{B}{V_3(2)+\eps} \\
    &=A\left(
        \frac1{V_1(2)+\eps}
        -\frac1{V_2(2)+\eps}
      \right)
      +B\left(
        \frac1{V_1(2)+\eps}
        -\frac1{V_3(2)+\eps}
      \right)
    \stackrel{\eqref{vineq2}}{<}0.
\end{align*}
This strict negativity holds for every $b,q\in[0,1)$ and every $\eps\ge0$.  The remaining details follow those of Theorem \ref{thm:all-betas}, e.g. observing that $V_1(a),V_2(a),V_3(a)$ are continuous functions of $a$ by \eqref{vmaxdef}, so $S_{\rm AdaMax}(a)<0$ for $a$ near $2$, and so on.
\end{proof}

\section{Muon}\label{secmuon}

We define the Muon optimization method and extend Theorem \ref{thm:main} to it. 
The Muon optimization method is defined for functions of matrices $f_t\colon\mathcal K\to\R$ $\forall$ $t\geq1$ where $\mathcal K\subset\R^{n\times n}$.  The iterations satisfy
\[
    m_t=b m_{t-1}+(1-b)\nabla f_t(x_t),
    \qquad
    x_{t+1}
    =
    \Pi_{\mathcal K}
    \left(
        x_t-\alpha_t\operatorname{Polar}(m_t)
    \right),\qquad\forall\,t\geq1,
\]
with \(m_0=0\), and $\Pi_{\mathcal K}$ denoting the projection to the nearest point in $\mathcal{K}$, with respect to the Euclidean (Frobenius) metric on $\R^{n\times n}$.  Here $m_t, x_{t}\in\R^{n\times n}$ for all $t\geq1$ and
\[
    \operatorname{Polar}(A)\colonequals A(A^* A)^{-1/2}
\]
is defined for any invertible $n\times n$ matrix $A$ \cite{par26}.  More generally, $\operatorname{Polar}(A)\colonequals UV$ when $A$ is an $n\times n$ matrix with reduced singular value decomposition $A=UDV$ (noting that the product $UV$ is well-defined even though $U,V$ are not uniquely determined by $A$).  Also $\operatorname{Polar}(0)\colonequals0$.  In practice, $\operatorname{Polar}(A)$ can be approximated by Newton-Schulz iterations.

In the case $n=1$ with $\mathcal K \colonequals[-1,1]\subset\R$, Muon becomes a signed momentum method:
\[
    m_t=b m_{t-1}+(1-b)f_t'(x_t),
    \qquad
    x_{t+1}
    =
    \Pi_{[-1,1]}
    \left(
        x_t-\alpha_t\mathrm{sign}(m_t)
    \right),\qquad\forall\,t\geq1,
\]
(Here $\mathrm{sign}(0)\colonequals0$.)  We will demonstrate this method has nonzero average regret.  The $n=1$ example can then be extended to the $n>1$ case by choosing each $f_t$ to be a function of one diagonal entry of its input matrix, e.g. $\mathcal K=\{\mathrm{diag}(x,0,\ldots,0)\colon x\in[-1,1]\}$.

\begin{theorem}[Muon Counterexample]
Let $\alpha_t>0$ satisfy
$\lim_{t\to\infty}\alpha_t=0$, $\lim_{t\to\infty}\frac{\alpha_{t+1}}{\alpha_t}=1$ and $\sum_{t=1}^{\infty}\alpha_t=\infty$.  Fix $b\in[0,1)$.  Consider Muon on the domain $\mathcal F=[-1,1]$.
Then there exists $\delta>0$, depending only on $b$, such that for
the linear functions \eqref{ftdef}, 
the iterates of Muon satisfy
\[
    \lim_{t\to\infty}x_t =1.
\]
However, the best fixed comparator is $x^*=-1$, and
\[
    \lim_{T\to\infty}\frac{R_T}{T}
    =
    \frac{2\delta}{3}.
\]
\end{theorem}

\begin{proof}

By Lemma \ref{lemma1}, the steady-state momentum values are
\[
    M_1(a)=\frac{a-b-b^2}{1+b+b^2},
\qquad
    M_2(a)=\frac{ab-b^2-1}{1+b+b^2},
\qquad
    M_3(a)=\frac{ab^2-b-1}{1+b+b^2}.
\]
At \(a=2\),
\[
    M_1(2)>0,\qquad M_2(2)<0,\qquad M_3(2)<0.
\]
Therefore, by continuity, there exists \(\delta>0\) such that, with
\(a=2+\delta\),
\[
    M_1(a)>0,\qquad M_2(a)<0,\qquad M_3(a)<0.
\]
By Lemma \ref{lemma1}, $(m_{3k+1},m_{3k+2},m_{3k+3})$ converges
exponentially to this period-three steady-state as $k\to\infty$.  Hence, there is some $k_0>0$ such that, for all $k>k_0$,
\[
    \operatorname{sign}(m_{3k+1})=+1,
    \qquad
    \operatorname{sign}(m_{3k+2})=-1,
    \qquad
    \operatorname{sign}(m_{3k+3})=-1.
\]

Thus the three unprojected scalar increments in the \(k\)th period are
\[
    u_{k,1}=-\alpha_{3k+1},
    \qquad
    u_{k,2}=\alpha_{3k+2},
    \qquad
    u_{k,3}=\alpha_{3k+3}.
\]
Since $\lim_{t\to\infty}\alpha_{t+1}/\alpha_t=1$, their sum satisfies
\[
    U_k
    \colonequals
    -\alpha_{3k+1}
    +
    \alpha_{3k+2}
    +
    \alpha_{3k+3}
    =
    \alpha_{3k+1}(1+o(1)).
\]

Denote $z_{k}\colonequals x_{3k+1}$.  Lemma \ref{lem:projection} as used for Adam then gives
\[
    z_{k+1}
    \ge
    \min(1,z_k+\alpha_{3k+1}/2),\qquad\forall\,k\geq k_0.
\]
Since $z_k\le1$ for all $k\geq1$ and $\sum_{k\ge k_0}\alpha_{k}=\infty$, we have $\sum_{k\geq k_0}\alpha_{3k+1}=\infty$, which follows since $\lim_{k\to\infty}\alpha_{k+1}/\alpha_{k}=1$.  So, the sequence
 \((z_k)\) reaches \(1\) at some finite value of $k$.  The within-period moves have size $O(\alpha_{3k})=o(1)$ by assumption, so
\(x_t\to1\) as $t\to\infty$.

Finally, every period has cumulative gradient
\[
    (2+\delta)-1-1=\delta>0.
\]
Hence the best fixed comparator in $[-1,1]$ is $-1$.  Since $\lim_{t\to\infty}x_t =1$, the same scalar regret computation as for Adam gives
\[
    \lim_{T\to\infty}\frac{R_T}{T}
    =
    \frac{2\delta}{3}.
\]
\end{proof}

\appendix
\section{Adam with i.i.d. Slopes}\label{secapp}

We now give a stochastic version of the Adam counterexample from Theorem \ref{thm:main}.  For simplicity, we only consider the step size $\alpha_t\colonequals\alpha/\sqrt{t}$ for all $t\geq1$.  Instead of presenting
the gradients in the deterministic period-three order \(a,-1,-1\), we draw
them independently at each time $t$.  The slope \(a\) appears with probability
\(1/3\), and the slope \(-1\) appears with probability \(2/3\).

More formally, let \((X_t)_{t\ge1}\) be i.i.d. Bernoulli random variables with
\[
    \mathbb P(X_t=1)=\frac13,\qquad
    \mathbb P(X_t=0)=\frac23,\qquad\forall\,t\geq1.
\]
For any \(a>0\), define
\[
    g_t(a)\colonequals-1+(a+1)X_t,\qquad\forall\,t\geq1.
\]
Thus, for any $t\geq1$,
\[
    g_t(a)=a \quad\text{with probability }1/3,
    \qquad
    g_t(a)=-1 \quad\text{with probability }2/3.
\]
The optimized functions are again
\[
    f_t(x)=g_t(a)x,\qquad \forall\,x\in[-1,1],\,t\geq1.
\]

\begin{theorem}[Adam i.i.d. random-slope counterexample]\label{iidthm}
Fix \(b,q\in[0,1)\), \(\varepsilon\ge0\), and \(\alpha>0\).  There exists
\(\delta>0\), depending only on \(b,q,\varepsilon\), such that, with
\[
    a=2+\delta,
\]
projected Adam on \([-1,1]\), driven by the i.i.d. slopes
\[
    g_t(a)=
    \begin{cases}
        a, & \text{with probability }1/3,\\
        -1, & \text{with probability }2/3,
    \end{cases}
\]
satisfies
\[
    \lim_{t\to\infty}x_t=1
    \qquad\text{almost surely}.
\]
Moreover, the best fixed comparator is eventually \(-1\), and the average
regret satisfies
\[
    \lim_{T\to\infty}\frac{R_T}{T}
    =
    \frac{2\delta}{3}
    \qquad\text{almost surely}.
\]
\end{theorem}

\begin{proof}
We first analyze the balanced value \(a=2\).  It is convenient to work with
a two-sided i.i.d. extension \((X_t)_{t\in\Z}\).  For
\(\lambda\in[0,1)\), define the stationary weighted average
\begin{equation}\label{atdef}
    A_{\lambda,t}
    \colonequals(1-\lambda)\sum_{\ell=0}^{\infty}\lambda^\ell X_{t-\ell}.
\end{equation}
Then
$
    \mathbb E[A_{\lambda,t}]=\frac13.
$
For \(a=2\), we have, for all $t\in\Z$
\[
    g_t(2)=-1+3X_t,
    \qquad
    g_t(2)^2=1+3X_t.
\]
The stationary versions of the first and second Adam moments are therefore
\[
    m_t^*(2)
    =
    (1-b)\sum_{\ell=0}^{\infty}b^\ell g_{t-\ell}(2)
    =
    -1+3A_{b,t},
    \qquad
    v_t^*(2)
    =
    (1-q)\sum_{\ell=0}^{\infty}q^\ell g_{t-\ell}(2)^2
    =
    1+3A_{q,t}.
\]
Define the stationary normalized Adam direction
\[
    h_t^*(2)
    \colonequals
    \frac{m_t^*(2)}
    {\sqrt{v_t^*(2)}+\varepsilon}
    =
    \frac{-1+3A_{b,t}}
    {\sqrt{1+3A_{q,t}}+\varepsilon}.
\]

\textbf{Step 1. Proving a positive drift.}  In the proof of Theorem \ref{thm:all-betas}, we used \eqref{eq:delta-choice} to show that $S(2)<0$ and $S(a)<0$ for $a$ near $2$.  In the current proof, the analogous statement is that $h_t^*(2)$ has negative mean.  That is, we claim that the following expression does not depend on $t\in\Z$ and
\begin{equation}\label{mu2ineq}
    \mu(2)\colonequals\mathbb E[h_t^*(2)]<0.
\end{equation}
Let
$
    \phi(y)\colonequals\frac{1}{\sqrt{1+3y}+\varepsilon},
$
$\forall$ $y\geq0$.
Then \(\phi\) is strictly decreasing on \([0,1]\).  Since
\(\mathbb E[A_{\lambda,t}]=1/3\),
\[
    \mu(2)
    =
    \mathbb E\left[(-1+3A_{b,t})\phi(A_{q,t})\right]  
=  3\,\mathrm{Cov}\!\left(A_{b,t},\phi(A_{q,t})\right).
\]
The random variable \(A_{b,t}\) is an increasing function of the coordinates
\((X_t,X_{t-1},X_{t-2},\ldots)\).  The random variable \(A_{q,t}\) is also
an increasing function of these same coordinates, and therefore
\(\phi(A_{q,t})\) is a decreasing function of them.  By the Harris
correlation inequality for product measures,
\[
    \mathrm{Cov}\!\left(A_{b,t},\phi(A_{q,t})\right)\le0.
\]
The inequality is in fact strict.  For any $t\geq1$ define the $\sigma$-algebra
\[
\mathcal G_t
\colonequals\sigma(X_{t-1},X_{t-2},\ldots)
\]
and define $B_t,Q_t$ so that
\[
A_{b,t}=B_t+(1-b)X_t,
\qquad
A_{q,t}=Q_t+(1-q)X_t,
\]
and such that \(B_t\) and \(Q_t\) are \(\mathcal G_t\)-measurable.  Set
\(p\colonequals\mathbb P(X_t=1)=1/3\).  Conditional on \(\mathcal G_t\),
\[
\mathbb E[A_{b,t}\mid\mathcal G_t]
   =B_t+p(1-b),
\qquad
\mathbb E[\phi(A_{q,t})\mid\mathcal G_t]
   =(1-p)\phi(Q_t)+p\phi(Q_t+1-q),
\]
\[
\operatorname{Cov}
 \bigl(A_{b,t},\phi(A_{q,t})\mid\mathcal G_t\bigr)
 =
 p(1-p)(1-b)
 \bigl[\phi(Q_t+1-q)-\phi(Q_t)\bigr]  
 <0,
\]
since \(b,q<1\) and \(\phi\) is strictly decreasing. Moreover,
$
\mathbb E[A_{b,t}\mid\mathcal G_t]
$
is an increasing function of the coordinates
\((X_{t-1},X_{t-2},\ldots)\), whereas
$
\mathbb E[\phi(A_{q,t})\mid\mathcal G_t]
$
is a decreasing function of those coordinates.  The Harris
correlation inequality, now applied to the product measure of the
past coordinates, therefore gives
\[
\operatorname{Cov}\!\left(
 \mathbb E[A_{b,t}\mid\mathcal G_t],
 \mathbb E[\phi(A_{q,t})\mid\mathcal G_t]
 \right)\le 0.
\]
The law of total covariance consequently yields
$
\operatorname{Cov}(A_{b,t},\phi(A_{q,t}))<0.
$
Thus
\[
\mu(2)=3\operatorname{Cov}(A_{b,t},\phi(A_{q,t}))<0.
\]

Now consider general \(a\) near \(2\).  Since for all $t\in\Z$
\[
    g_t(a)=-1+(a+1)X_t,
\qquad
    g_t(a)^2=1+(a^2-1)X_t,
\]
the stationary moments are
\begin{equation}\label{stdef}
    m_t^*(a)
    =
    -1+(a+1)A_{b,t},
\qquad
    v_t^*(a)
    =
    1+(a^2-1)A_{q,t}.
\end{equation}
Define
\[
    h_t^*(a)
    \colonequals
    \frac{m_t^*(a)}
    {\sqrt{v_t^*(a)}+\varepsilon},
\qquad
    \mu(a)\colonequals\mathbb E[h_t^*(a)].
\]
The denominator is bounded away from zero and the integrand is bounded and
continuous in \(a\) in a neighborhood of \(2\).  Hence, by dominated
convergence, \(a\mapsto\mu(a)\) is continuous.  Since \(\mu(2)<0\), there
exists \(\delta>0\) such that, with
$
    a=2+\delta,
$
we still have
\begin{equation}\label{muan}
    \mu(a)<0.
\end{equation}
Decreasing \(\delta\) if necessary, assume also \(\delta\le1\).  

Set
$
    \gamma\colonequals-\mu(a)>0.
$
Thus the stationary Adam direction has negative mean:
$
    \mathbb E[h_t^*(a)]=-\gamma.
$
Equivalently, the mean update direction
$
    -h_t^*(a)
$
is positive.

We next transfer this stationary drift to the actual Adam process initialized
at \(m_0=v_0=0\).  The actual moments $m_t$ satisfy
\[
    m_t=(1-b)\sum_{\ell=0}^{t-1}b^\ell g_{t-\ell}(a),\qquad\forall\,t\geq1,
\]
while the stationary version is
\[
    m_t^*(a)=(1-b)\sum_{\ell=0}^{\infty}b^\ell g_{t-\ell}(a),\qquad\forall\,t\in\Z.
\]
Since the gradients $g_t$ satisfy $|g_t|\leq3$ for all $t$,
\[
    |m_t-m_t^*(a)|\le 3b^t,
\qquad
    |v_t-v_t^*(a)|\le 9q^t.
\]
Also \(v_t^*(a)\ge1\) for all $t\in\Z$ and, for the actual process,
\[
    v_t\stackrel{\eqref{adameq}}{=}
    qv_{t-1}+(1-q)g_t^2\stackrel{\eqref{adameq}}{\ge} 1-q>0,\qquad\forall\,t\geq1.
\]
Therefore
$
    h_t\colonequals\frac{m_t}{\sqrt{v_t}+\varepsilon}
$
satisfies
\begin{equation}\label{htineq}
    |h_t-h_t^*(a)|\le C\rho^t,\qquad\forall\,t\geq1,
\end{equation}
for some \(C<\infty\) and \(\rho\in(0,1)\).

\textbf{Step 2.}  We now prove the following tail excursion estimate.  This technical Lemma has no determinstic analogue, i.e. it was not needed in the proof of Theorem \ref{thm:all-betas}.

\begin{lemma}[Weighted positive-drift estimate]\label{lem:weighted-positive-drift}
Let $(X_t)_{t\in\mathbb Z}$ be i.i.d. random variables, and let
$(Z_t)_{t\in\mathbb Z}$ be a bounded stationary Bernoulli shift of the
form
\[
    Z_t\colonequals F(X_t,X_{t-1},X_{t-2},\ldots),\qquad\forall\,t\in\Z,
\]
where $F$ is real-valued.  Assume that there are constants $C_0<\infty$ and $\rho\in(0,1)$ such that,
for every $\ell\ge0$, changing only the coordinate $X_{t-\ell}$ can change
$Z_t$ by at most $C_0\rho^\ell$.  Suppose
\begin{equation}\label{zbardef}
    \bar z\colonequals\mathbb E Z_t>0.
\end{equation}
Let $\alpha>0$.  Let $\alpha_t\colonequals\alpha/\sqrt t$ for all $t\geq1$.  Then as $T\to\infty$,
\[
    \sum_{t=1}^T \alpha_t Z_t\to +\infty
    \qquad\text{almost surely},
\]
and the uniform tail adverse excursion satisfies: as $n\to\infty$,
\begin{equation}\label{dnineq}
    \Delta_n
    \colonequals
    \sup_{n\le r\le s<\infty}
    \max\Big(-\sum_{t=r}^s\alpha_tZ_t,\,0\Big)
    \longrightarrow 0
    \qquad\text{almost surely}.
\end{equation}
\end{lemma}

\begin{proof}
The weighted strong law follows from the  ergodic theorem by Abel
summation.  Let
\[
    S_n\colonequals\sum_{t=1}^n Z_t,
    \qquad
    c_n\colonequals\frac{S_n}{n}.
\]
By the ergodic theorem, $c_n\to\bar z$ almost surely as $n\to\infty$.  Write
$w_t=t^{-1/2}$ for all $t\geq1$ and $W_T=\sum_{t=1}^T w_t$ for all $T\geq1$.  Summation by parts gives
\[
    \sum_{t=1}^T w_t Z_t
    =
    w_T S_T+
    \sum_{t=1}^{T-1}(w_t-w_{t+1})S_t.
\]
Since $S_t=tc_t$,
\[
    \frac{\sum_{t=1}^T w_t Z_t}{W_T}
    =
    \frac{Tw_T}{W_T}c_T
    +
    \sum_{t=1}^{T-1}
    \frac{t(w_t-w_{t+1})}{W_T}c_t.
\]
%
The coefficients on the right are nonnegative and sum to one since
\[
Tw_T
    +\sum_{t=1}^{T-1}t(w_t - w_{t+1})
=\sum_{t=1}^{T}w_t = W_T.
\]
So, the Toeplitz lemma implies that, as $T\to\infty$,
\begin{equation}\label{wsl}
    \frac{\sum_{t=1}^T w_t Z_t}{\sum_{t=1}^T w_t}	\to \bar z
    \qquad\text{almost surely}.
\end{equation}
Since $\sum_{t=1}^T\alpha_t\sim2\alpha\sqrt T$, this implies as $T\to\infty$
\begin{equation}\label{atinf}
    \sum_{t=1}^T\alpha_t Z_t\to+
    \infty
    \qquad\text{almost surely}.
\end{equation}

It remains to prove \eqref{dnineq}.
Fix \(\eta>0\), and for \(N\ge 1\) let
\[
I_N\colonequals\{N,N+1,\ldots,2N-1\}.
\]
For an interval \([r,s]\subseteq I_N\), set
\begin{equation}\label{mrsdef}
m\colonequals s-r+1,
\qquad
Y_{r,s}\colonequals \sum_{t=r}^s \alpha_t Z_t.
\end{equation}

We will apply McDiarmid's inequality to \(Y_{r,s}\), regarded
as a function of the independent coordinates \((X_j)_{j\le s}\).
By assumption, changing only \(X_j\) can change \(Y_{r,s}\) by at most
\[
d_j\colonequals
C_0\sum_{t=\max\{r,j\}}^s
       \alpha_t\rho^{\,t-j},
\qquad j\le s,
\]
where an empty sum is interpreted as zero.  Since \(\alpha_t\le \alpha/\sqrt N\) for \(t\in I_N\), we have
\[
d_j\le
\frac{C_0\alpha}{(1-\rho)\sqrt N},\qquad\forall\,r\leq j\leq s,
\]
\[
d_j\le
\frac{C_0\alpha}{(1-\rho)\sqrt N}\rho^{\,r-j},\qquad\forall\,j<r.
\]
It follows that
\[
\sum_{j=-\infty}^s d_j^2
 \le
 \frac{C_1(m+1)}{N}
\]
for a constant \(C_1<\infty\) depending only on
\(C_0,\rho,\alpha\).

McDiarmid's inequality therefore gives
\[
\mathbb P\bigl(Y_{r,s}-\mathbb EY_{r,s}\le-u\bigr)
 \le
 \exp\left(
  -\frac{2u^2}{\sum_{j=-\infty}^s d_j^2}
 \right),\qquad\forall\,u>0.
\]
One may justify its use for the countable family
\((X_j)_{j\le s}\) by first fixing all coordinates before a finite
time, applying the finite-dimensional inequality, and then sending
that time to \(-\infty\).  The assumed summable coordinate
sensitivities make the corresponding approximation uniform.  Since \(t\le 2N\) on \(I_N\) and $[r,s]\subseteq I_N$,
\[
\mathbb EY_{r,s}
\stackrel{\eqref{mrsdef}\wedge\eqref{zbardef}}{=}\overline z\sum_{t=r}^s\alpha_t
 \stackrel{\eqref{mrsdef}}{\geq}
 \frac{\alpha\overline z}{\sqrt{2N}}\,m
 \equalscolon c_0\frac{m}{\sqrt N}\stackrel{\eqref{zbardef}}{>}0.
\]
Consequently, using $-\eta\colonequals-u+\mathbb{E}Y_{r,s}$, for $N$ sufficiently large,
\[
\mathbb P(Y_{r,s}\le-\eta)
\le
\exp\Big(
 -c_1
 N\big(\eta+c_0m/\sqrt N\big)^2/(m+1)
\Big)                                    
\le
\exp(-c_\eta\sqrt N).
\]
For the last inequality, use that $\eta>0$ for $N$ sufficiently large by definition of $\eta$,
\[
\big(\eta+c_0m/\sqrt N\big)^2
 \ge 2\eta c_0m/\sqrt N,
\]
and \(m/(m+1)\ge 1/2\) by \eqref{mrsdef}.  Taking a union bound over the at most \(N^2\) intervals with integer endpoints in \(I_N\)
gives, for all $N$ sufficiently large
\[
\P\Big(
 \exists\, r,s\in I_N,\ r\le s\colon
 \sum_{t=r}^s\alpha_tZ_t\le-\eta
\Big)
\le
N^2\exp(-c_\eta\sqrt N).
\]
This bound is summable along the dyadic sequence \(N=2^j\).
Therefore, by the Borel--Cantelli lemma, almost surely there exists
\(j_0(\eta)\) such that, for every \(j\ge j_0(\eta)\), every
subinterval of the dyadic block
\[
B_j\colonequals\{2^j,\ldots,2^{j+1}-1\}
\]
has weighted sum greater than \(-\eta\).

We also claim that every sufficiently late dyadic block has
positive weighted sum.  Define
\[
A_T\colonequals\sum_{t=1}^T\alpha_tZ_t.
\]
The weighted strong law \eqref{wsl} proven above gives
\[
A_T
 =\overline z\sum_{t=1}^T\alpha_t+o(\sqrt T)
 \qquad\text{almost surely}.
\]
Hence
\[
A_{2^{j+1}-1} - A_{2^{j}-1}
=\sum_{t=2^j}^{2^{j+1}-1}\alpha_tZ_t
 =
 \overline z\sum_{t=2^j}^{2^{j+1}-1}\alpha_t
 +o(2^{j/2})
 >0
\]
for all sufficiently large \(j\), since
$
\sum_{t=2^j}^{2^{j+1}-1}\alpha_t\asymp 2^{j/2}.
$

Now consider an interval \([r,s]\) lying sufficiently far in the
tail.  If it is contained in one dyadic block, its weighted sum is
greater than \(-\eta\).  Otherwise, decompose it into a terminal
piece of its first dyadic block, a collection of complete dyadic
blocks, and an initial piece of its final dyadic block.  The complete dyadic
blocks have nonnegative sum, and each of the two boundary pieces has
sum greater than \(-\eta\) (almost surely, for sufficiently large $j$).  Thus
\[
\sum_{t=r}^s\alpha_tZ_t>-2\eta
\]
for every sufficiently late interval \([r,s]\).

Finally, apply the preceding argument simultaneously to the
countable sequence \(\eta=1/k\), \(k\ge1\).  On the resulting
probability-one event, for every \(k\) we have
\[
\Delta_n\le \frac{2}{k}
\]
for all sufficiently large \(n\).  Therefore as $n\to\infty$,
\[
\Delta_n\longrightarrow0
\qquad\text{almost surely}.
\]
\end{proof}

\textbf{Step 3.  Applying Lemma~\ref{lem:weighted-positive-drift}}.  Returning to the proof of Theorem \ref{iidthm}, we will apply Lemma~\ref{lem:weighted-positive-drift}, so we verify its coordinate-sensitivity assumption.  Let
\begin{equation}\label{ztdef}
Z_t^*\colonequals-h_t^*(a)
      =-\frac{m_t^*(a)}{\sqrt{v_t^*(a)}+\varepsilon}.
\end{equation}
From \eqref{stdef} and \eqref{atdef}, changing only \(X_{t-\ell}\) changes the stationary moment $m_t^*(a)$ by at most
\[
\left|\Delta m_t^*(a)\right|
 =(a+1)(1-b)b^\ell
\]
and similarly $v_t^*(a)$ changes by at most
\[
\left|\Delta v_t^*(a)\right|
 =(a^2-1)(1-q)q^\ell.
\]
On the region
$\{(m,v)\in\R^2\colon
|m|\le a,\, 1\le v\le a^2\},
$
the function
$
F(m,v)\colonequals-\frac{m}{\sqrt v+\varepsilon}
$
satisfies
\[
\left|\frac{\partial F}{\partial m}\right|
 \le 1,
\qquad
\left|\frac{\partial F}{\partial v}\right|
 =
 \frac{|m|}
 {2\sqrt v(\sqrt v+\varepsilon)^2}
 \le \frac a2.
\]
It then follows from \eqref{ztdef} that the change in $Z_t^*$ from changing $X_{t-\ell}$ is at most
\[
\begin{aligned}
|\Delta Z_t^*|
&\le
 (a+1)(1-b)b^\ell
 +\frac a2(a^2-1)(1-q)q^\ell.
\end{aligned}
\]
Choose any
$
\rho_0\in\bigl(\max(b,q),1\bigr).
$
Then there is a constant \(C_0<\infty\) such that
\[
|\Delta Z_t^*|\le C_0\rho_0^\ell,
\qquad \ell\ge0.
\]
Moreover, \(Z_t^*\) is bounded, since for all $2\leq a\leq3$
\[
|m_t^*(a)|\le a,\qquad v_t^*(a)\ge1.
\]
Thus \(Z_t^*\) satisfies all the hypotheses of Lemma~\ref{lem:weighted-positive-drift}.

We apply Lemma~\ref{lem:weighted-positive-drift} first to the stationary
process $Z_t^*$ from \eqref{ztdef}.  It has positive mean by \eqref{muan}.  The actual process $Z_t\colonequals-h_t$ differs from $Z_t^*$ by an exponentially
decaying error by \eqref{htineq}:
$
    |Z_t-Z_t^*|\le C\rho^t.
$
Consequently,
\[
    \sum_{t=1}^\infty \alpha_t |Z_t-Z_t^*|<\infty,
    \qquad
    \lim_{n\to\infty}\sup_{n\le r\le s<\infty}
    \sum_{t=r}^s\alpha_t|Z_t-Z_t^*|	=0.
\]
Thus the two conclusions of Lemma~\ref{lem:weighted-positive-drift} also
hold for $Z_t=-h_t$.

\textbf{Step 4.  Proving $\lim_{t\to\infty}x_t =1$}.  We now prove that $\lim_{t\to\infty}x_t =1$.  The update is
\[
x_{t+1}
 =\Pi_{[-1,1]}(x_t+\alpha_tZ_t),\qquad\forall\,t\geq1.
\]

Suppose there exists $m$ such that $x_t<1$ for all $t>m$.  Then for any $n>m$ we have
$$1>x_{n+1}\geq x_{m}+\sum_{t=m}^{n}\alpha_t Z_t,$$
since projection on $[-1,1]$ can only increase $x_{t}<1$ when $t>m$.  But this contradicts \eqref{atinf}.  We therefore conclude that $x_t=1$ for infinitely many times $t$.

Fix \(n\), and let \(\tau\ge n\) be a time such that \(x_\tau=1\).  For all \(m\ge\tau\), let
\[
r_m\colonequals\max\{r\in[\tau,m]\colon x_r=1\}.
\]
If \(r_m=m\), then \(1-x_m=0\).  Otherwise, none of the iterates
\[
x_{r_m+1},\ldots,x_m
\]
equals \(1\).  Hence projection to $1$ does not occur for the indices from \(r_m\) through \(m-1\).  Projection to $-1$ can only increase the
iterate, so an induction gives
\[
x_m
 \ge
 1+\sum_{t=r_m}^{m-1}\alpha_tZ_t.
\]
Since $x_m\leq1$, we have $-\sum_{t=r_m}^{m-1}\alpha_tZ_t\geq 0$, so for all $m\geq\tau\geq n$,
\[
1-x_m
 \le
 -\sum_{t=r_m}^{m-1}\alpha_tZ_t
 =
 \max\Big(
  -\sum_{t=r_m}^{m-1}\alpha_tZ_t
 ,\,0\Big)
 \stackrel{\eqref{dnineq}}{\le} \Delta_n.
\]
Since \(\Delta_n\to0\) almost surely by \eqref{dnineq}, for any \(\eta>0\) we may first
choose \(n\) so large that \(\Delta_n<\eta\), and then choose a
hitting time \(\tau\ge n\).  The preceding bound gives
\[
0\le1-x_m<\eta,
\qquad\forall\, m\ge\tau.
\]
Therefore
\[
\lim_{t\to\infty}x_t =1
\qquad\text{almost surely}.
\]

\textbf{Step 5.  Regret bound}.  It remains to compute regret.  By the strong law of large numbers and since $a=2+\delta$,
\[
    \lim_{T\to\infty}
    \frac1T\sum_{t=1}^T g_t(a)
    =\mathbb E[g_t(a)]
    =
    \frac{a}{3}-\frac23
    =
    \frac{\delta}{3}
    \qquad\text{almost surely}.
\]
Let
$
    G_T\colonequals\sum_{t=1}^T g_t(a).
$
Then $G_T/T\to\delta/3>0$, and hence $G_T>0$ for all sufficiently large
$T$, almost surely.  Therefore the best fixed comparator is eventually
$x_T^*=-1$, and for all sufficiently large $T$,
$
    R_T
    =
    \sum_{t=1}^T g_t(a)x_t+G_T.
$
Equivalently,
\[
    R_T
    =
    2G_T+\sum_{t=1}^T g_t(a)(x_t-1).
\]
Since $x_t\to1$ almost surely and the gradients $g_t$ are uniformly bounded,
\[
    \lim_{T\to\infty}\frac1T\sum_{t=1}^T g_t(a)(x_t-1)=0
    \qquad\text{almost surely}.
\]
Combining this with $\lim_{T\to\infty}G_T/T=\delta/3$, we obtain
\[
    \lim_{T\to\infty}\frac{R_T}{T}=\frac{2\delta}{3}
    \qquad\text{almost surely}.
\]

\end{proof}

\noindent\textbf{Acknowledgement.}  ChatGPT 5.5 assisted in the preparation of this manuscript.

\bibliographystyle{amsalpha}

\begin{thebibliography}{99}

\bibitem[AMM+20]{ala20}
Ahmet Alacaoglu, Yura Malitsky, Panayotis Mertikopoulos, and Volkan Cevher.
\emph{A new regret analysis for Adam-type algorithms}.
International Conference on Machine Learning (2020), 119, pp. 202--210.

\bibitem[AZK+24]{ahn2024adamftrl}
Kwangjun Ahn, Zhiyu Zhang, Yunbum Kook, and Yan Dai.
\newblock \emph{Understanding Adam optimizer via online learning of updates: Adam is FTRL in disguise}.
\newblock International Conference on Machine Learning (2024).

\bibitem[BG20]{dasilva2020ode}
Andr\'e Belotto da Silva and Maxime Gazeau.
\newblock \emph{A general system of differential equations to model first order adaptive algorithms}.
\newblock Journal of Machine Learning Research (2020), 21 (129), pp. 1--42.


\bibitem[BMR+20]{brown20}
Tom B. Brown, Benjamin Mann, Nick Ryder, Melanie Subbiah, Jared Kaplan, Prafulla Dhariwal, Arvind Neelakantan, Pranav Shyam, Girish Sastry, Amanda Askell, Sandhini Agarwal, Ariel Herbert-Voss, Gretchen Krueger, Tom Henighan, Rewon Child, Aditya Ramesh, Daniel M. Ziegler, Jeffrey Wu, Clemens Winter, Christopher Hesse, Mark Chen, Eric Sigler, Mateusz Litwin, Scott Gray, Benjamin Chess, Jack Clark, Christopher Berner, Sam McCandlish, Alec Radford, Ilya Sutskever, and Dario Amodei.
\emph{Language models are few-shot learners}.
Advances in Neural Information Processing Systems
(2020) 159, pp. 1877--1901.


\bibitem[BW19]{bockweiss2022adam}
Sebastian Bock and Martin Georg Wei{\ss}.
\newblock \emph{Non-convergence and limit cycles in the Adam optimizer}.
International Conference on Artificial Neural Networks (2019), vol 11728.


\bibitem[BZZ+26]{bai2026degenerate}
Zhiwei Bai, Jiajie Zhao, Zhangchen Zhou, Zhi-Qin John Xu, and Yaoyu Zhang.
\newblock \emph{Towards understanding Adam convergence on highly degenerate polynomials}.
International Conference on Machine Learning (2026), to appear.


\bibitem[D24]{deep25}
DeepSeek-AI.
\emph{DeepSeek-V3 Technical Report}.  (2024), Preprint, \href{https://arxiv.org/abs/2412.19437v1}{arXiv:2412.19437}.


\bibitem[DCK+19]{devlin19}
Jacob Devlin, Ming-Wei Chang, Kenton Lee, and Kristina Toutanova.
\emph{BERT: Pre-training of Deep Bidirectional Transformers for Language Understanding}. North American Chapter of the Association for Computational Linguistics (2019), pp. 4171--4186.


\bibitem[DDJ+25]{dereich2025symmetry}
Steffen Dereich, Thang Do, Arnulf Jentzen, and Philippe von Wurstemberger.
\newblock \emph{Adam symmetry theorem: characterization of the convergence of the stochastic Adam optimizer}.
(2025), Preprint, \href{https://arxiv.org/abs/2511.06675}{arXiv:2511.06675}.

\bibitem[DGJ24]{dereich2024nonvanishing}
Steffen Dereich, Robin Graeber, and Arnulf Jentzen.
\newblock \emph{Non-convergence of Adam and other adaptive stochastic gradient descent optimization methods for non-vanishing learning rates}.
(2024), Preprint, \href{https://arxiv.org/abs/2407.08100}{arXiv:2407.08100}.


\bibitem[DHJ24]{dohannibaljentzen2024relu}
Thang Do, Sonja Hannibal, and Arnulf Jentzen.
\newblock \emph{Non-convergence to global minimizers in data driven supervised deep learning: Adam and stochastic gradient descent optimization provably fail to converge to global minimizers in the training of deep neural networks with ReLU activation}.
Journal of Mathematical Analysis and Applications (2026), 130724.

\bibitem[DJR25]{dojentzenriekert2025risk}
Thang Do, Arnulf Jentzen, and Adrian Riekert.
\newblock \emph{Non-convergence to the optimal risk for Adam and stochastic gradient descent optimization in the training of deep neural networks}.
(2025), Preprint, \href{https://arxiv.org/abs/2503.01660}{arXiv:2503.01660}.

\bibitem[HWD19]{huang2019}
Haiwen Huang, Chang Wang, and Bin Dong.  \emph{Nostalgic Adam: Weighting More of the Past Gradients When Designing the Adaptive Learning Rate}. International Joint Conference on Artificial Intelligence (2019), pp. 2556--2562.


\bibitem[JR25]{jentzenriekert2025global}
Arnulf Jentzen and Adrian Riekert.
\newblock \emph{Non-convergence to global minimizers for Adam and stochastic gradient descent optimization and constructions of local minimizers in the training of artificial neural networks}.
\newblock SIAM/ASA Journal on Uncertainty Quantification (2025), 13 (3), pp. 1294--1333.

\bibitem[KB15]{kingma2015adam}
Diederik P. Kingma and Jimmy Ba.
\newblock \emph{Adam: A method for stochastic optimization}.
\newblock International Conference on Learning Representations (2015). (Poster)


\bibitem[LH19]{los19}
Ilya Loshchilov and Frank Hutter.
\emph{Decoupled Weight Decay Regularization}. International Conference on Learning Representations (2019) (Poster).

\bibitem[LXL+19]{luo2019adabound}
Liangchen Luo, Yuanhao Xiong, Yan Liu, and Xu Sun.
\newblock \emph{Adaptive gradient methods with dynamic bound of learning rate}.
\newblock International Conference on Learning Representations (2019).


\bibitem[PKC+26]{par26}
Tetiana Parshakova, Ahmed Khaled, Michael Crawshaw, Guillaume Garrigos, and Robert M. Gower.
\emph{Muon Does Not Converge on Convex Lipschitz Functions}. (2026), 
Preprint, \href{https://arxiv.org/abs/2605.08980}{arXiv:2605.08980}.


\bibitem[RKK18]{reddi2018adam}
Sashank J. Reddi, Satyen Kale, and Sanjiv Kumar.
\newblock \emph{On the convergence of Adam and beyond}.
\newblock International Conference on Learning Representations (2018).


\bibitem[Toi23]{toint2023adam}
Philippe L. Toint.
\newblock \emph{Divergence of the ADAM algorithm with fixed-stepsize: a (very) simple example}.
(2023), Preprint, \href{https://arxiv.org/abs/2308.00720}{arXiv:2308.00720}.

\bibitem[TMS+23]{touvron23}
	Hugo Touvron, Louis Martin, Kevin Stone, Peter Albert, Amjad Almahairi, Yasmine Babaei, Nikolay Bashlykov, Soumya Batra, Prajjwal Bhargava, Shruti Bhosale, Dan Bikel, Lukas Blecher, Cristian Canton-Ferrer, Moya Chen, Guillem Cucurull, David Esiobu, Jude Fernandes, Jeremy Fu, Wenyin Fu, Brian Fuller, Cynthia Gao, Vedanuj Goswami, Naman Goyal, Anthony Hartshorn, Saghar Hosseini, Rui Hou, Hakan Inan, Marcin Kardas, Viktor Kerkez, Madian Khabsa, Isabel Kloumann, Artem Korenev, Punit Singh Koura, Marie-Anne Lachaux, Thibaut Lavril, Jenya Lee, Diana Liskovich, Yinghai Lu, Yuning Mao, Xavier Martinet, Todor Mihaylov, Pushkar Mishra, Igor Molybog, Yixin Nie, Andrew Poulton, Jeremy Reizenstein, Rashi Rungta, Kalyan Saladi, Alan Schelten, Ruan Silva, Eric Michael Smith, Ranjan Subramanian, Xiaoqing Ellen Tan, Binh Tang, Ross Taylor, Adina Williams, Jian Xiang Kuan, Puxin Xu, Zheng Yan, Iliyan Zarov, Yuchen Zhang, Angela Fan, Melanie Kambadur, Sharan Narang, Aurélien Rodriguez, Robert Stojnic, Sergey Edunov, and Thomas Scialom.
\emph{Llama 2: Open Foundation and Fine-Tuned Chat Models}. (2023), Preprint, \href{https://arxiv.org/abs/2307.09288}{arXiv:2307.09288}.


\bibitem[WK22]{wang2022vradam}
Ruiqi Wang and Diego Klabjan.
\newblock \emph{Divergence results and convergence of a variance reduced version of Adam}.
(2022), Preprint, \href{https://arxiv.org/abs/2210.05607}{arXiv:2210.05607}.

\bibitem[XZL+24]{xie24}
Xingyu Xie, Pan Zhou, Huan Li, Zhouchen Lin, and Shuicheng Yan.
\emph{Adan: Adaptive Nesterov Momentum Algorithm for Faster Optimizing Deep Models}. IEEE Transactions on Pattern Analysis and Machine Intelligence (2024), 46 (12), pp. 9508--9520.


\bibitem[ZCS+22]{zhang2022adam}
Yushun Zhang, Congliang Chen, Naichen Shi, Ruoyu Sun, and Zhi-Quan Luo.
\newblock \emph{Adam can converge without any modification on update rules}.
\newblock Advances in Neural Information Processing Systems (2022).


\bibitem[ZLC+26]{zhang2026adam}
Yushun Zhang, Bingran Li, Congliang Chen, Zhi-Quan Luo, and Ruoyu Sun.
\newblock \emph{Adam converges without any modification on update rules}.
(2026), Preprint, \href{https://arxiv.org/abs/2603.02092}{arXiv:2603.02092}.


\bibitem[ZRS+18]{zaheer2018yogi}
Manzil Zaheer, Sashank J. Reddi, Devendra Sachan, Satyen Kale, and Sanjiv Kumar.
\newblock \emph{Adaptive methods for nonconvex optimization}.
\newblock Advances in Neural Information Processing Systems (2018).

\end{thebibliography}

\end{document}